\documentclass[accepted]{uai2026} 
                        
 \usepackage[british]{babel}

\usepackage{natbib} 
\usepackage{mathtools} 
\usepackage{booktabs} 
\usepackage{xcolor}
\usepackage{tikz} 
\usetikzlibrary{arrows.meta, calc, positioning}

\usepackage{bbm}
\usepackage{tcolorbox}
\usepackage{amsmath}
\usepackage{amssymb}
\usepackage{amsthm}
\usepackage{enumitem}
\usepackage{siunitx}
\usepackage{pifont}
\usepackage{algorithm}
\usepackage{algpseudocode}
\usepackage{colortbl}
\usepackage{bm}
\usepackage{titletoc} 
\usepackage{stackengine,scalerel}
\usepackage{tabularx,ragged2e}
\usepackage{multirow}
\newcolumntype{Y}{>{\RaggedRight\arraybackslash}X}
 \usepackage{float}  

\definecolor{linegrey}{RGB}{100, 100, 100}
\definecolor{blue}{RGB}{0, 0, 128}
\definecolor{red}{RGB}{180, 20, 20} 
 
\newcommand{\R}{\mathbb{R}}

\newcommand{\EE}{\mathbb{E}}       

\newcommand{\cF}{\mathcal{F}}      
\newcommand{\cX}{\mathcal{X}}      
\newcommand{\cM}{\mathcal{M}}      
\newcommand{\cY}{\mathcal{Y}}      
\newcommand{\cV}{\mathcal{V}}      

\newcommand{\cD}{\mathcal{D}}

\newcommand{\cQ}{\mathcal{Q}}
\newcommand{\cE}{\mathcal{E}}

\newcommand{\ksimplex}{\Delta_{K-1}}

\newcommand{\dy}{\delta_{y}} 
\newcommand{\on}[1]{\operatorname{#1}}

\newcommand{\AU}{\on{AU}_\cF}
\newcommand{\EU}{\on{EU}_\cF}

\DeclareMathOperator*{\argmax}{arg\,max}

\newcommand{\dF}[1]{d_\cF\!\left( #1 \right)}

\newcommand{\dataset}[1]{\textsc{\lowercase{\mbox{#1}}}}
\newcommand{\datasetlabel}[1]{%
  \fcolorbox{black!60}{white}{\strut\bfseries #1}%
}

 \newcommand{\abox}[1]{%
  \tikz[baseline=(X.base)]\node
    (X)
    [draw,
     rounded corners=0.6pt,
     line width=0.3pt,
     inner xsep=1.6pt,
     inner ysep=0.1pt]
    {\strut #1};%
}

\newcommand{\Cint}{\ThisStyle{\ensurestackMath{%
  \stackinset{c}{.25\LMpt}{c}{.5\LMpt}{\SavedStyle\mathrm{c}}{\SavedStyle\int}}}}

\newcommand{\dd}{\mathrm{d}}
    
\theoremstyle{plain}
\newtheorem{theorem}{Theorem}[section]
\newtheorem{proposition}[theorem]{Proposition}

\newtheorem{corollary}[theorem]{Corollary}
\theoremstyle{definition}

\theoremstyle{remark}
\newtheorem{remark}[theorem]{Remark}
 
\title{Quantification of Credal Uncertainty: A Distance-Based Approach}

\author[1]{\href{mailto:<xabier.gonzalez@unavarra.es>?Subject=UAI 2026 paper}{Xabier González-García}{}}
\author[2]{Siu Lun Chau}
\author[3,4]{Julian Rodemann}
\author[5,6]{Michele Caprio}
\author[3]{Krikamol Muandet}
\author[1]{Humberto Bustince}
\author[7]{Sébastien Destercke}
\author[8,9,10]{Eyke Hüllermeier}
\author[8,9]{Yusuf Sale}

\affil[1$\bullet$2]{%
Dept.\ of Statistics, CS \& Mathematics,
Public Univ.\ of Navarre $\bullet$
EPIC Lab,
Nanyang Technological Univ. Singapore
}

\affil[3$\bullet$4]{%
Rational Intelligence Lab, CISPA Helmholtz Center for Information Security   $\bullet$ 
Dept.\ of Statistics, LMU Munich
  }
\affil[5$\bullet$6]{%
Dept.\ of Computer Science, The University of Manchester  $\bullet$ 
Manchester Centre for AI Fundamentals
  }
\affil[7$\bullet$8]{%
Universit\'{e} de Technologie de Compi\`{e}gne, CNRS, Heudiasyc $\bullet$ 
Institute of Informatics, LMU Munich  
  }
\affil[9$\bullet$10]{%
Munich Center for Machine Learning (MCML)  $\bullet$ 
German Research Center for Artificial Intelligence (DFKI, DSA)
  }

\begin{document}
    \maketitle   

\begin{abstract}
Credal sets, i.e., closed convex sets of probability measures, provide a natural framework to represent aleatoric and epistemic uncertainty in machine learning. Yet how to \emph{quantify} these two types of uncertainty for a given credal set, particularly in multiclass classification, remains underexplored.
In this paper, we propose a distance-based approach to quantify total, aleatoric, and epistemic uncertainty for credal sets.
Concretely, we introduce a family of such measures within the framework of Integral Probability Metrics (IPMs).
The resulting quantities admit clear semantic interpretations, satisfy natural theoretical desiderata, and remain computationally tractable for common choices of IPMs.
We instantiate the framework with the total variation distance and obtain simple, efficient uncertainty measures for multiclass classification. In the binary case, this choice recovers established uncertainty measures, for which a principled multiclass generalization has so far been missing.
Empirical results confirm the practical usefulness of the proposed measures,
showing competitive performance at low computational cost.
\end{abstract}  

\section{Introduction}\label{sec:introduction}
\begin{figure*}[hbt!]  
    \centering
    \includegraphics[width=\textwidth]{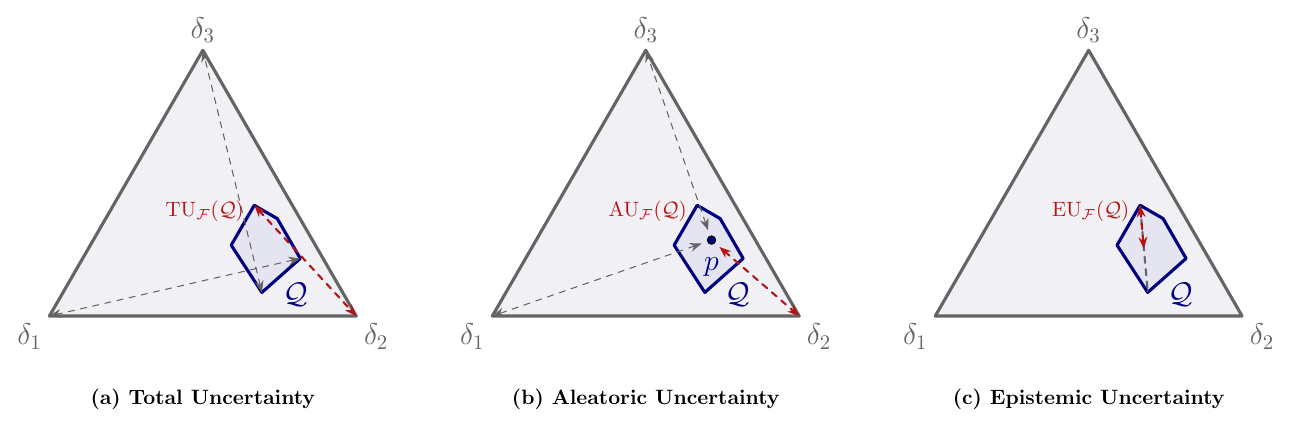}
    \caption{Geometric illustration of the proposed distance-based framework on the simplex $\Delta_{K-1}$ ($K=3$). Dirac measures \textcolor{linegrey}{$\{\delta_1,\delta_2,\delta_3\}$} represent full certainty.
\textbf{(a) Total uncertainty:} distance of $\cQ$ to full certainty, measured as the worst-case distance to the nearest vertex \textcolor{linegrey}{$\delta_y$}.
\textbf{(b) Aleatoric uncertainty:} the distance of a precise predictive distribution $p$ from full certainty, measured by its proximity to the nearest vertex. For a credal prediction $\cQ$, aleatoric uncertainty is set-valued, giving a range over $p\in\cQ$.
\textbf{(c) Epistemic uncertainty:} the imprecision of the credal set \(\cQ\), quantified as half its maximal diameter, i.e., the largest distance between any two distributions in \(\cQ\). 
Distances (\textcolor{red}{\(\longleftrightarrow\)}) are defined via IPMs over $\cF$.}
\label{fig:framework}
 
\end{figure*}
The growing impact of machine learning in societal and scientific applications has intensified the need for predictive models that are not only accurate but also uncertainty-aware. Central to this discussion is the distinction between \emph{aleatoric} and \emph{epistemic} uncertainty \citep{hullermeier2021aleatoric, sale2026aleatoric}. The former refers to variability that is inherent in the data-generating process and hence irreducible, whereas the latter reflects limited knowledge about the predictive mechanism.
This distinction is also operationally important, as only epistemic uncertainty is reducible\,---\,for instance, by collecting more data or improving the model.

Accordingly, a broad range of methods has been developed to learn predictive models that can represent \emph{both} types of uncertainty, often by means of higher-order formalisms such as second-order distributions~\citep{gal2016dropout, lakshminarayanan2017simple, malinin2018predictive, sensoy2018evidential}, or through the framework of \textit{imprecise probability}~\citep{walley1991statistical, augustin2014introduction}, which extends classical probability theory beyond the Kolmogorov axioms.
\emph{Credal sets}, i.e., (convex) sets of probability measures~\citep{levi1980enterprise, walley1991statistical}, offer a natural and intuitive alternative for machine learning~\citep[cf.][]{caprio2024credalb, wang2025credal}: randomness within each probability measure captures aleatoric uncertainty, while the extent of the set itself reflects epistemic uncertainty.

In this paper, we address the complementary task of \emph{(credal) uncertainty quantification}: given a prediction in the form of a credal set, we seek summaries (\emph{viz.} uncertainty measures) that capture the extent of predictive uncertainty. 
A central and nontrivial challenge is \emph{disaggregation}, i.e., how to decompose (total) uncertainty into its aleatoric and epistemic components. 
The literature offers numerous measures of total, aleatoric, and epistemic uncertainty for credal sets  \citep[e.g.][]{abellan2000non,abellan2005difference,abellan2006disaggregated}, but most proposals either lack a clear semantic interpretation of the decomposed components — often relying on an assumed additive decomposition — or are computationally too costly for use in large-scale machine learning pipelines.
In this spirit, \cite{hullermeier2022quantification} provide a critical discussion of uncertainty measures and highlight several pitfalls that arise when such notions are used for credal uncertainty quantification, especially in machine learning settings.
Moreover, they propose measures that satisfy natural axiomatic desiderata and whose theoretical expectations align well with empirical behavior, for instance in selective prediction experiments. 
A key limitation, however, is that their treatment is restricted to the binary case, and a principled extension to multiclass classification has so far been missing.
More generally, credal uncertainty quantification in the multiclass setting has received comparatively little attention in machine learning, arguably due to additional conceptual and computational challenges \citep{sale2023volume}. This gap motivates the present work.

\emph{Contributions.} 
We present a principled framework for multiclass credal uncertainty quantification that yields a family of measures of total, aleatoric, and epistemic uncertainty with clear semantics and strong computational properties (Figure~\ref{fig:framework}).
Our approach is distance-based, partly inspired by \citet{sale2024second}, but formulated for credal sets rather than second-order distributions, and grounded in Integral Probability Metrics (IPMs)~\citep{muller1997integral}.
IPMs encompass many widely used distances between probability measures, including total variation, Wasserstein distance, and maximum mean discrepancy, and let practitioners pick the one best suited to the task at hand.
We instantiate the framework with the total variation distance, which yields particularly simple and efficient measures for multiclass classification. In the binary case, this choice recovers the uncertainty measures of \citet{hullermeier2022quantification}. 
Empirical results on standard benchmarks show that the resulting measures are competitive with existing baselines, thereby making them well tailored to machine learning needs.

\emph{Paper organization.}
In Section~\ref{sec:prelim}, we introduce the learning setting and notation. 
In Section~\ref{sec:uq}, we review existing approaches to credal uncertainty quantification, discuss their limitations, and position our proposal in relation to recent work. 
In Section~\ref{sec:novel}, we present our distance-based framework and derive a novel family of measures for total, aleatoric, and epistemic uncertainty. 
In Section~\ref{sec:exp}, we report empirical results, and in Section~\ref{sec:conlc}, we conclude.
Complementary results can be found in the supplementary material.

\section{Preliminaries}\label{sec:prelim}
We consider supervised classification with instance space $\cX$, finite label set $\cY=\{1,\dots,K\}$, and i.i.d.\ training data $\cD_\text{train}=\{(x^{(n)},y^{(n)})\}_{n=1}^N$ drawn from an (unknown) distribution on $\cX\times\cY$.
A probabilistic predictor is a mapping $x \mapsto p_x \in \ksimplex$, where $\ksimplex$ is the probability simplex over $\cY$; thus, each $p_x$ is a probability measure on $(\cY, 2^{\cY})$. Our focus is predictive uncertainty at a query $x_q$, i.e., uncertainty about the predicted label $\hat y_q$. A single probabilistic predictor naturally accounts for aleatoric uncertainty, but leaves uncertainty about the ground truth data-generating process itself implicit.
To make this explicit, we move from point predictions in $\ksimplex$ to \emph{set-valued} predictions: at a query $x_q$, the predictor returns a \emph{credal set} $\cQ_{x_q}\in\cV(\ksimplex)$, where $\cV(\ksimplex)$ denotes the family of nonempty closed convex subsets of the simplex. We drop the input subscript when the context is clear, writing simply $p$ and $\cQ$.
Given a credal set $\cQ$ and $A\subseteq\cY$, its lower and upper probabilities are
\begin{align*}
\underline{p}_{\cQ}(A) = \inf_{p \in \cQ} p(A),
\qquad
\overline{p}_{\cQ}(A)= \sup_{p \in \cQ} p(A),
\end{align*}
which summarize the range of plausible probabilities assigned to $A$ \citep{walley1991statistical}.
Often, we consider \emph{finitely generated} credal predictions, i.e., sets of the form 
$
\cQ = \mathrm{conv}(\{p^{(1)},\dots,p^{(M)}\}),
$
where $\{p^{(m)}\}_{m=1}^M$ is a finite collection of probabilistic models, and $\mathrm{conv}(\cdot)$ denotes the convex hull.
For finitely generated credal sets, lower and upper probabilities can be expressed directly in terms of the generators: since $p\mapsto p(A)$ is linear, the extrema over $\cQ$ are attained at extreme points, hence (writing $[M]:=\{1,\dots,M\}$)
\begin{align*}
\underline{p}_{\cQ}(A) = \min_{m\in[M]} p^{(m)}(A),
\qquad
\overline{p}_{\cQ}(A) = \max_{m\in[M]} p^{(m)}(A).
\end{align*}

\section{Uncertainty Quantification}\label{sec:uq}
Recent critiques of uncertainty quantification methods have highlighted that some widely used approaches fail to satisfy basic properties one would reasonably expect \citep{wimmer2023quantifying}.
An axiomatic perspective makes such failures visible by specifying the
qualitative behavior of a measure.
Following \citet{abellan2005additivity, jirouvsek2018new}, we consider a generic measure $f:\mathcal{V}(\ksimplex)\to\mathbb{R}$ that maps a credal set $\cQ$ to a scalar, and we recall a collection of common axioms for uncertainty measures:
\begin{itemize}[leftmargin=2.2em, labelsep=0.6em]
    \item[\abox{A1}] Boundedness: There exists $C>0$ such that we have $f(\cQ)\in[0,C]$ for all credal sets $\cQ$.
    \item[\abox{A2}] Continuity: The functional $f$ is continuous\footnote{Throughout,  \(\mathcal V(\Delta_{K-1})\) is endowed with the Hausdorff topology induced  by any metric generating weak convergence on \(\Delta_{K-1}\). Since  \(\mathcal Y\) is finite, this is equivalently the Hausdorff topology induced  by the Euclidean metric on the simplex.}.
    \item[\abox{A3}] Monotonicity: If $\cQ \subseteq \cM$, then $f(\cQ) \leq f(\cM)$.
    \item[\abox{A4}] Probability consistency: If $\cQ = \{p\}$ is precise, then $f(\cQ)$ reduces to a functional of the single distribution $p$, i.e., the measure carries no residual set-level structure when the credal set is a singleton.
\end{itemize}
These axioms are regularity and consistency requirements. In particular, monotonicity and probability consistency formalize that enlarging the set of plausible distributions should not reduce uncertainty and that the framework agrees with the precise case. 
Moreover, machine learning applications often impose computational constraints, especially in multiclass settings with many labels.
This motivates complementing the above desiderata with an additional requirement.
\begin{itemize}[leftmargin=2.2em, labelsep=0.6em]
    \item[\abox{A5}] Extreme-point characterization: Let $\operatorname{ext}(\cQ)$ denote the  set of extreme points of $\cQ$. A scalar uncertainty functional is called  extreme-point evaluable if its value can be computed by optimizing only  over $\operatorname{ext}(\cQ)$, rather than over all of $\cQ$.
\end{itemize}
Axiom A5 formalizes an invariance principle with respect to convexification. 
When a credal set is specified via a set of generators and then closed under convex combinations, A5 ensures that the uncertainty measure depends only on the extreme points and is unaffected by the particular interior representation of $\cQ$.
This property can substantially simplify computation by reducing evaluation of $f(\cQ)$ to $\mathrm{ext}(\cQ)$, for example when $\cQ$ is induced by a finite collection of predictive models such as an ensemble of probabilistic classifiers.

A1--A5 together form a set of desiderata for scalar credal uncertainty  functionals in machine learning. 
Such an axiomatic perspective is useful because empirical evaluation of uncertainty is necessarily indirect: before assessing a measure in downstream tasks, it clarifies which qualitative behavior the measure can be expected to exhibit.
The desiderata apply directly to scalar total and  epistemic uncertainty measures. Aleatoric uncertainty is different in our  framework, because its primary quantification is set-valued; scalar  aleatoric scores arise only after applying a summary operator. Consequently, the axioms satisfied by an aleatoric summary  depend on this additional choice. We therefore state explicitly, for each  proposed quantity and summary, which of A1--A5 are satisfied.
We deliberately do not promote additional, more specialized properties to the status of formal axioms. Additional useful properties are mentioned in context throughout the paper, but are not enforced as basic desiderata.
The literature also considers additional axioms such as subadditivity on joint credal sets and additivity under (imprecise) probabilistic independence \citep{abellan2005additivity, bronevich2008axioms}. These properties concern uncertainty under composition and depend on notions of (imprecise) independence. In this work, we focus on predictive uncertainty quantification and therefore do not consider such joint-set conditions as basic requirements.

\subsection{Existing Measures}
In the following, we briefly review prominent existing approaches to credal uncertainty quantification. We focus on measures most commonly used in machine learning practice. For each family, we highlight key limitations that motivate our distance-based framework. For additional critical discussion, see e.g. \citet{hullermeier2022quantification, sale2023volume}.
\subsubsection{Entropy Measures}
Arguably the most prominent measure of uncertainty for a (precise) probability measure $p \in \ksimplex$ is the Shannon entropy 
\citep{shannon1948mathematical}, 
$S(p) = -\sum_{y \in \mathcal{Y}} p(y) \log_2 p(y).$
Entropy is minimal when $p$ assigns all probability mass to a single label and maximal when $p$ is the uniform distribution on $\cY$, while also enjoying well-known axiomatic and operational justifications \citep{csiszar2008axiomatic}.
Shannon entropy can be lifted from precise distributions to credal sets by taking its upper and lower envelopes over $\cQ$ \citep{abellan2005difference, abellan2006disaggregated}. 
Concretely, define
\begin{align*}
S^{\ast}(\cQ) \coloneqq \sup_{p\in\cQ} S(p),
\qquad
S_{\ast}(\cQ) \coloneqq \inf_{p\in\cQ} S(p).
\end{align*}
The upper envelope $S^{\ast}(\cQ)$ is commonly interpreted as a measure of total uncertainty, while the lower envelope $S_{\ast}(\cQ)$ captures a notion of aleatoric uncertainty. Assuming an additive decomposition, epistemic uncertainty is then quantified by the residual gap $S^{\ast}(\cQ) - S_{\ast}(\cQ).$

\emph{Criticism.} The upper envelope $S^{\ast}(\cQ)$ is known to satisfy the classical axioms (A1)--(A4) under standard assumptions. 
From a machine learning perspective, however, its practical use is limited, since computing $S^{\ast}(\cQ)$ has no closed-form solution and requires iterative solvers at prediction time. 
Moreover, apart from falling off as a ``residual'' from the additive decomposition of total uncertainty, $S^{\ast}(\cQ)-S_{\ast}(\cQ)$ is lacking a theoretical (axiomatic) grounding and clear semantic interpretation \citep{hullermeier2022quantification}. 

\subsubsection{Hartley Measures}
In contrast, set-theoretic approaches model uncertainty by defining a subset $A \subseteq \cY$ of possible outcomes, distinguishing between what is plausible and what is ruled out.
A standard way to quantify the uncertainty associated to the set $A$ is Hartley’s non-specificity \citep{hartley1928transmission}, $H(A) = \log_2 |A|$, which depends only on how many outcomes remain possible.
To lift Hartley-type non-specificity to credal sets, one can combine $\log_2|A|$ with a set function derived from the lower probability $\underline{p}_{\cQ}$ via a Möbius transform \citep{abellan2000non}. Concretely, define 
\begin{align*}
GH(\cQ) := \sum_{\emptyset\neq A\subseteq\mathcal Y} \log_2 |A|\,m_Q(A), 
\end{align*}
where $m_Q(A):= \sum_{B\subseteq A}(-1)^{|A\setminus B|}\underline p_Q(B)$ is the Möbius inverse of $\underline p_Q$. The empty set is omitted from the sum, since Hartley's non-specificity is defined for nonempty sets and $m_Q(\emptyset)=\underline p_Q(\emptyset)=0$.
In the literature, $\on{GH}(\cQ)$ is commonly interpreted as a measure of epistemic uncertainty.
Total uncertainty can again be quantified by $S^\ast(\cQ)$. Assuming an additive decomposition, one then defines the corresponding aleatoric component as the residual $S^{\ast}(\cQ) - \on{GH}(\cQ)$.

\emph{Criticism.} The quantity $\on{GH}(\cQ)$ captures a notion of \emph{imprecision} by quantifying, in a set-theoretic manner, how many alternatives remain plausible, largely independent of the shape of the distributions in $\cQ$. 
It vanishes for precise predictions ($\cQ=\{p\}$) and satisfies several desirable axiomatic properties in the imprecise probability literature \citep{abellan2000non}. 
From a machine learning perspective, however, its practical use is limited by unfavorable scaling. Evaluation requires summation over the power set $2^{\cY}$ and thus becomes prohibitive even for moderate $K$.
Moreover, the associated aleatoric component defined via the residual $S^{\ast}(\cQ)- \on{GH}(\cQ)$ inherits the general drawback of additive disaggregation schemes, namely that it is specified by subtraction rather than by an independent semantic principle.

\subsubsection{Other Approaches}
Recently, \emph{distance-based} notions of uncertainty \citep[e.g.][]{sale2024second, apostolopoulou2024rate, chau2025integral} have attracted attention in the literature.
For instance, \citet{sale2024second} develop uncertainty measures for second-order distributions by comparing a distribution-over-distributions to a reference notion of certainty via probability metrics. Our work is aligned with this perspective, but is formulated for credal predictions and thus conceptually extends the distance-based idea to credal sets.

In a related line of work, \citet{chau2025integral} study extensions of IPMs~\citep{muller1997integral} to imprecise models.
An IPM is a discrepancy between probability measures defined through a class $\cF$ of measurable test functions (typically bounded, and often chosen to enforce additional regularity such as Lipschitz continuity). 
For two probability measures $p$ and $q$ on a measurable space $(\cY,\mathcal{A})$, the IPM is
\begin{align*} 
d_{\cF}(p,q)
\; \coloneqq\;
\sup_{f \in \cF}\Bigl|\; \int f\,\dd p \;-\; \int f\,\dd q \;\Bigr|.
\end{align*}
Intuitively, $d_{\cF}(p,q)$ measures how well $p$ and $q$ can be distinguished by tests in $\cF$, namely by the largest discrepancy in expected function values.
Many standard distances are special cases of IPMs, obtained by choices of the test-function class $\cF$.
To extend IPM-type discrepancies beyond precise probabilities, \citet{chau2025integral} replace the Lebesgue integral by the Choquet integral, which integrates a function with respect to a capacity (a non-additive set function). This yields a class of Integral Imprecise Probability Metrics (IIPMs). Within this framework, they propose the Maximum Mean Imprecision (MMI) to quantify epistemic uncertainty. While MMI is formulated for capacities, it can be adapted to credal sets by evaluating induced lower and upper probabilities. Concretely, one considers 
\begin{align*}
d_{\cF}\!\bigl(\underline p_{\cQ},\,\overline p_{\cQ}\bigr)
\;\coloneqq\;
\sup_{f\in\cF}\Bigl(
\Cint f\,\dd \overline p_{\cQ}
\;-\;
\Cint f\,\dd \underline p_{\cQ}
\Bigr),
\end{align*}
where $\Cint f\,\dd\nu$ denotes the Choquet integral of $f$ with respect to the set function $\nu$ \citep{choquet1954theory, denneberg1994non}.

\emph{Criticism.} 
First, passing from a credal set $\cQ$ to its lower and upper probabilities $(\underline p_{\cQ},\overline p_{\cQ})$ is not injective, since distinct credal sets can induce the same envelopes, so envelope-based quantities may fail to distinguish different credal sets.
Second, evaluating Choquet-integral-based discrepancies for general set functions can be computationally demanding. For instance, \citet{chau2025integral} propose computable upper bounds to mitigate this intractability for the total variation instantiation. 
Recently, \citet{chau2026quantifying} extend MMI to plausibility measures and derive closed-form expressions computable in linear time. However, this result applies to plausibility measures, a restricted subclass of upper probabilities, and thus does not cover the general credal set setting.

Alternative notions of total variation for non-additive set functions have also been proposed recently in the imprecise-probability literature \citep{nieto2025total, nieto2025imprecise,vskulj2013coefficients}. Additional discussion of related work is provided in \autoref{app:related}.

\section{A Family of Novel Measures}\label{sec:novel}
As we highlighted in the preceding section, current approaches to credal uncertainty quantification have limitations.
They often rely on additive decompositions, so the resulting components lack direct semantic meaning, or scale poorly.  
In this section, we propose a distance-based framework for credal uncertainty quantification (see Figure~\ref{fig:framework}).

\paragraph{Total uncertainty.}
We quantify total uncertainty as the distance of a credal set $\cQ$ from full certainty. 
Fully certain predictions correspond to the vertices of the probability simplex, i.e., the Dirac measures
\begin{align*}
\cE(\ksimplex) = \{\,\delta_y : y \in \cY,\ \delta_y(A)=\mathbf{1}_{\{y\}}(A)\,\},
\end{align*}
each representing complete certainty about a single class.
Since proximity to one vertex implies distance from the others, we take the nearest Dirac measure as the \emph{reference point} of zero uncertainty.
Accordingly, we define total uncertainty as the minimal point-to-set (Hausdorff-type) worst-case distance between $\cQ$ and the set of Dirac measures:
\begin{equation}\label{eq:tu}
\on{TU}_\cF(\cQ) \;:=\; \inf_{y \in \cY}\, \sup_{p \in \cQ} \dF{p,\dy}
\end{equation}
The worst-case choice ensures that no other source of uncertainty exceeds the \emph{total} uncertainty.
\paragraph{Aleatoric uncertainty.}
For credal predictions $\cQ$, aleatoric uncertainty is generally not uniquely determined, since different $p \in \cQ$ may exhibit different levels of randomness. 
Consequently, any scalar quantification of aleatoric uncertainty for $\cQ$ reflects an additional modeling (or summarization) choice and necessarily loses information. 
Yet, scalar summaries are often needed in downstream applications that require a total ordering or thresholding of predictions.

We first consider the precise setting and then lift this notion to credal predictions.
For a single predictive distribution $p$, aleatoric uncertainty quantifies its intrinsic randomness and is obtained as the distance to full certainty, i.e., to the nearest Dirac measure. Equivalently, this is the special case of \eqref{eq:tu} in which the credal set reduces to $\{p\}$:
\begin{equation}\label{eq:au}
\AU(p) := \inf_{y \in \cY} \dF{p,\dy}.
\end{equation}

Rather than forcing a scalar summary, we therefore define aleatoric uncertainty on the quantification level as the \emph{set} of all plausible randomness values induced by elements of $\cQ$:
\begin{equation}\label{eq:AUQ}
\AU(\cQ) \;:=\; \bigl\{ \AU(p) \;:\; p \in \cQ \bigr\}.
\end{equation}
This set-valued quantification is not merely a formal convenience, but a principled consequence of the credal representation.
Indeed, any direct aggregation of \eqref{eq:AUQ} into a single scalar implicitly selects or combines elements of $\cQ$ and thereby resolves part of the uncertainty that the credal set is meant to represent. 

To make the set-valued quantification operational in downstream tasks, we explicitly separate the quantification level from the summary level. 
Let $\mathfrak{A}$ denote the class of all aleatoric uncertainty sets attainable from credal predictions, i.e.,
\[
\mathfrak{A} \;\coloneqq\; \bigl\{ \AU(\cQ) \;:\; \cQ \subseteq \ksimplex  \bigr\}.
\]
A (task-dependent) summary operator is a mapping
\begin{equation*}
\Phi:\mathfrak{A}\to\mathcal{Z},
\end{equation*}
where $\mathcal{Z}$ is an application-specific output space (e.g., $\R$, $\R^2$, or another ordered space). 
Given such a $\Phi$, we define the corresponding actionable aleatoric summary by
\begin{equation*}
\AU^\Phi(\cQ) \;\coloneqq\; \Phi\bigl(\AU(\cQ)\bigr).
\end{equation*}
This formalizes that any scalar (or vector-valued) aleatoric score used in practice is not part of quantification itself, but results from an additional summarization choice.

A summary operator $\Phi$ is called \emph{lossless} (on $\mathfrak{A}$) if it is injective, i.e., if
\begin{equation*}
\Phi(A)=\Phi(B) \;\Longrightarrow\; A=B
\qquad\text{for all } A,B\in\mathfrak{A}.
\end{equation*}
In that case, $\Phi(\AU(\cQ))$ retains the full information contained in the set-valued aleatoric uncertainty $\AU(\cQ)$.

A canonical choice in our setting is the endpoint summary
\begin{equation*}
\Phi_{\mathrm{int}}(A) \;\coloneqq\; \bigl(\inf A,\sup A\bigr),
\end{equation*}
which yields $\on{AU}_{\cF}^{\Phi_{\mathrm{int}}}(\cQ) \;=\; \bigl(\underline{\on{AU}}_{\cF}(\cQ),\,\overline{\on{AU}}_{\cF}(\cQ)\bigr)$, where
\begin{equation*}
\underline{\on{AU}}_{\cF}(\cQ)\coloneqq \inf_{p\in\cQ}\on{AU}_{\cF}(p),
\quad
\overline{\on{AU}}_{\cF}(\cQ)\coloneqq \sup_{p\in\cQ}\on{AU}_{\cF}(p).
\end{equation*}
Under mild conditions (e.g., if $\cQ$ is convex and compact and $p\mapsto \AU(p)$ is continuous), the image $\AU(\cQ)\subseteq\R_{\ge 0}$ is a compact interval, so that
\begin{equation*}
\AU(\cQ)=\bigl[\underline{\on{AU}}_{\cF}(\cQ),\,\overline{\on{AU}}_{\cF}(\cQ)\bigr].
\end{equation*}
Hence, in this case, the endpoint summary is lossless and provides an exact finite-dimensional representation of the set-valued aleatoric uncertainty.

When a downstream task requires a total ordering (as in selective prediction), an ordering rule can be imposed on top of the endpoint summary (e.g., lexicographically), without changing the underlying set-valued aleatoric quantification.

\paragraph{Epistemic uncertainty.}
Our aim is to quantify the imprecision represented by a credal set. Intuitively, a larger credal set corresponds to greater ignorance about the data-generating mechanism. Although quantities like the volume of a credal set have been considered, they show limited effectiveness in multiclass classification tasks \citep{sale2023volume}.
We instead quantify epistemic uncertainty via the (half) maximal pairwise distance within the credal set, i.e., via half its maximal diameter \citep{chambers2007degree, stewart2022distention, caprio2024credal}:
\begin{equation}\label{eq:eu}
\begin{aligned}
\EU(\cQ)
&:= \frac{1}{2}\sup_{p,q \in \cQ} \, d_{\cF}(p,q) \\
&= \frac{1}{2}\sup_{f \in \cF}
\Bigl(
\sup_{p \in \cQ} \EE_p[f]
-
\inf_{q \in \cQ} \EE_q[f]
\Bigr).
\end{aligned}
\end{equation}
In contrast to volume-based notions, the diameter depends only on the maximal discrepancy among plausible predictive distributions, not on the overall size of the set.
A large diameter means that $\cQ$ contains at least two distributions with strongly conflicting predictions.
This disagreement captures ignorance about the ground truth and is independent of the intrinsic randomness of the individual elements of $\cQ$ (cf.\ Proposition~\ref{prop:eushift}).
The factor $\tfrac{1}{2}$ provides a convenient normalization and, in particular, ensures that epistemic uncertainty is bounded by total uncertainty (cf. Proposition~\ref{prop:tudominance}).

\subsection{Theoretical Properties}\label{sec:properties}
Since the axioms introduced earlier are formulated for scalar-valued functionals on credal sets, they apply directly to total and epistemic uncertainty.
By contrast, aleatoric uncertainty is quantified by the set-valued object~\eqref{eq:AUQ}.
To connect it to the axiomatic framework, we therefore consider scalar summaries of~\eqref{eq:AUQ}, focusing in particular on the canonical endpoint summaries $\underline{\on{AU}}_{\cF}(\cQ)$ and $\overline{\on{AU}}_{\cF}(\cQ)$. More generally, however, which axioms hold for the resulting aleatoric summary $\AU^\Phi(\cQ)$ depends on the choice of~$\Phi$.
\begin{remark}\label{rem:assumptions}
The properties established below rely on two standing assumptions on the function class~$\cF$:
First, we assume that $\cF$ is rich enough to separate probability measures, so that the IPM $d_{\cF}$ defines a metric on~$\ksimplex$ (rather than only a pseudo-metric).
For example, if $\cF=\cF_k$ is the unit ball of a reproducing kernel
Hilbert space with characteristic kernel $k$, then the corresponding maximum mean discrepancy is a metric on a suitable class of probability distributions; see, e.g., \citet[Sec.~3.3.1 and Sec.~3.5]{Muandet2017}.
Second, we assume that $\cF$ is a uniform class with respect to weak convergence~\citep{muller1997integral, rachev1991probability}, i.e., $\cF$ is equicontinuous and has uniformly bounded span,
\begin{equation*}
\sup_{f\in\cF}\bigl(\sup f-\inf f\bigr) < \infty .
\end{equation*}
By \citet[Theorem~4.3]{muller1997integral}, this ensures good continuity properties of the induced IPM.
On a finite label space~$\cY$, equicontinuity is automatic.
\end{remark}
\begin{proposition}\label{prop:properties}
Let $\cF$ be a uniform class with respect to weak convergence such that $d_{\cF}$ is a metric on~$\ksimplex$. Then $\on{TU}_\cF(\cQ)$ and $\on{EU}_\cF(\cQ)$ satisfy A1--A5; $\underline{\on{AU}}_{\cF}(\cQ)$ satisfies A1, A2, and A4; and $\overline{\on{AU}}_{\cF}(\cQ)$ satisfies A1--A4. 
\end{proposition}
\begin{corollary}\label{corollary:monotone}
Under the assumptions of Proposition~4.2, let $\cQ,\cM \subseteq \ksimplex$  be credal sets with $\cQ \subseteq \cM$. Then  
\[
\on{AU}_\cF(\cQ) \subseteq \on{AU}_\cF(\cM),
\]  
where $\on{AU}_\cF(\cQ) := \{\on{AU}_\cF(p) : p \in \cQ\}$.  In particular, whenever these sets are intervals,
\[
\underline{\on{AU}}_\cF(\cM)  \le  \underline{\on{AU}}_\cF(\cQ)  \le  \overline{\on{AU}}_\cF(\cQ)  \le  \overline{\on{AU}}_\cF(\cM).
\]
\end{corollary}
Hence, the endpoint interval is monotone with respect to interval inclusion;  equivalently, the upper endpoint is monotone in the scalar sense, whereas  the lower endpoint is anti-monotone.
We next show several structural properties that further strengthen the interpretation of the proposed framework.
First, total uncertainty is no smaller than epistemic uncertainty and any element of $\AU(\cQ)$.
\begin{proposition}\label{prop:tudominance}
Let $\cQ \subseteq \ksimplex$ be a credal set. Then, for every $p \in \cQ$,
\begin{equation*}
\EU(\cQ) \leq \on{TU}_\cF(\cQ)
\quad\text{and}\quad
\AU(p) \leq \on{TU}_\cF(\cQ).
\end{equation*}
\end{proposition}
Second, epistemic uncertainty does not depend on the location of the credal set in the simplex (i.e., it is translation-invariant within~$\ksimplex$).
\begin{proposition}\label{prop:eushift}
For any two credal sets $\cQ, \cQ' \subseteq \ksimplex$ that are translates of each other within the simplex,
\[
\EU(\cQ)=\EU(\cQ').
\]
\end{proposition}
Finally, we relate our epistemic quantification to MMI based
on the lower and upper probability envelopes induced by $\cQ$.
Unlike MMI, our measure operates directly on the credal set, while remaining related.
\begin{proposition}\label{prop:diammmi}
Let $\cQ \subseteq \Delta_{K-1}$ be a credal set. Then its
$\cF$-diameter is upper bounded by the MMI of its lower and upper
probabilities:
\begin{align*}
\sup_{p,q\in \cQ} d_{\cF}(p,q)
\le
d_{\cF}(\underline p_Q,\overline p_Q).
\end{align*}
Consequently,
\begin{align*}
\on{EU}_{\cF}(\cQ)
\le
\frac12 d_{\cF}(\underline p_Q,\overline p_Q).
\end{align*}
Equality holds whenever the lower expectation over $\cQ$ coincides with the
Choquet integral with respect to $\underline p_Q$, e.g., when
$\underline p_Q$ is 2-monotone.
\end{proposition}
Closed-form expressions and computational properties depend on the specific choice of~$\cF$.
In the next section, we specialize to the total variation test class, which satisfies the assumptions of Remark~\ref{rem:assumptions} on finite label spaces and yields explicit formulas for all proposed quantities.

\subsection{Total Variation Instantiation} \label{sec:closedforms}
We instantiate the proposed framework with total variation (TV), using the test class $\cF_{\mathrm{TV}} \;\coloneqq\; \{\mathbf 1_A : A\subseteq \cY\}$.
Then, for all $p,q\in\ksimplex$, the induced IPM coincides with the total variation distance
\[
d_{\cF_{\mathrm{TV}}}(p,q)
= \sup_{A\subseteq\cY}\bigl(p(A)-q(A)\bigr).
\]
This choice is natural in multiclass classification because it requires no additional geometric structure on the label space and yields the following simple closed-form expressions.
\begin{proposition}\label{prop:closedforms}
Let $\cQ \subseteq \ksimplex$ be a credal set. Then, instantiating \eqref{eq:tu}, \eqref{eq:au}, and \eqref{eq:eu}  with the test class $\cF=\cF_{\mathrm{TV}}$ yields the following closed-form expressions:
\begin{equation*}
\begin{aligned}
\on{TU}_{\cF_{\mathrm{TV}}}(\cQ)
&= 1 - \sup_{y \in \cY} \underline{p}_{\cQ}(\{y\}),\\[0.3em]
\on{AU}_{\cF_{\mathrm{TV}}}(p)
&= 1 - \sup_{y \in \cY} p(\{y\}),\\[0.3em]
\on{EU}_{\cF_{\mathrm{TV}}}(\cQ)
&= \frac{1}{4}\sup_{p,q\in\cQ}\sum_{y\in\cY}\bigl|p(\{y\})-q(\{y\})\bigr|.
\end{aligned}
\end{equation*}
\end{proposition}
\begin{remark}
Beyond total variation, alternative choices of $\cF$ tailor the semantics of the uncertainty measures.
For example, a Wasserstein--1 class (1-Lipschitz functions) incorporates a ground metric on $\cY$, so disagreements between nearby labels contribute less than between distant ones. This is particularly useful in ordinal, hierarchical, or cost-sensitive classification settings.
\end{remark}

For a precise model $p$, $\on{AU}_{\cF_{\mathrm{TV}}}(p)$ is the pointwise Bayes error under $0$--$1$ loss.
Moreover, $\on{TU}_{\cF_{\mathrm{TV}}}(\cQ)$ and $\on{AU}_{\cF_{\mathrm{TV}}}(p)$ depend only on singleton masses and can therefore be evaluated in linear time.
The epistemic component $\on{EU}_{\cF_{\mathrm{TV}}}(\cQ)$ is a maximal pairwise TV distance over $\cQ$ (equivalently, a maximal $\ell_1$ distance up to the factor $\tfrac12$), and for compact convex sets the supremum is attained at extreme points.
Hence, for finitely generated credal sets, it suffices to compute the maximum (TV) distance between generators.
\begin{proposition}\label{prop:AUTVextreme}
Let $\cQ \subseteq \ksimplex$ be a credal set and let $\mathrm{ext}(\cQ)$ denote its extreme points.
For $\cF=\cF_{\mathrm{TV}}$, the lower endpoint $\underline{\on{AU}}_{\cF_{\mathrm{TV}}}(\cQ)$ satisfies A5.
If there exists a class $y^\star\in\cY$ such that
$y^\star \in \argmax_{y\in\cY} p(\{y\})$ for every $p\in\mathrm{ext}(\cQ)$,
then $\overline{\on{AU}}_{\cF_{\mathrm{TV}}}(\cQ)$ is attained at an extreme point.
\end{proposition}
Table~\ref{tab:complexity} summarizes the resulting complexities and compares them with those of alternative uncertainty measures.
In addition to its computational appeal, the TV instantiation admits a connection to the MMI, as the next result shows.
\begin{corollary}\label{cor:mmiequiv}
For $\cF=\cF_{\mathrm{TV}}$, the following identity holds:
\begin{equation}\label{eq:mmiequiv}
\on{EU}_{\cF_{\mathrm{TV}}}(\cQ)
\;=\;
\frac{1}{2}\, d_{\cF_{\mathrm{TV}}}\!\bigl(\underline{p}_{\cQ},\,\overline{p}_{\cQ}\bigr).
\end{equation}
\end{corollary}
\begin{table}[t!]
\centering 
\begin{tabular}{@{}lll@{}}
\toprule
\textbf{Component} & \textbf{Measure} & \textbf{Complexity} \\
\midrule
\multirow{2}{*}{Total}
& Ours (TV)   & $\mathcal{O}(M K)$ \\
& Entropy     & $\mathcal{O}(T_{\mathrm{conv}} M K)$ \\
\midrule
\multirow{4}{*}{Aleatoric}
& Ours ($\overline{\on{AU}}_{\cF_{\mathrm{TV}}}$)
& $\Omega(MK)$, $\mathcal{O}(T_{\mathrm{lp}} M K)$ \\
& Ours ($\underline{\on{AU}}_{\cF_{\mathrm{TV}}}$)
& $\mathcal{O}(M K)$ \\
& Entropy
& $\mathcal{O}(M K)$ \\
& Hartley
& $\mathcal{O}(T_{\mathrm{conv}} M K + 3^K)$ \\
\midrule
\multirow{3}{*}{Epistemic}
& Ours (TV)
& $\mathcal{O}(M^2 K)$ \\
& Entropy
& $\mathcal{O}(T_{\mathrm{conv}} M K)$ \\
& Hartley
& $\mathcal{O}(3^K)$ \\
\bottomrule
\end{tabular}
\caption{Computational complexity of uncertainty measures for finitely generated credal sets, where $M$ is the number of generators, $K$ the number of classes, and $T_{\mathrm{conv}}$ / $T_{\mathrm{lp}}$ the iterations of a convex / linear optimization routine. For $\overline{\on{AU}}_{\cF_{\mathrm{TV}}}$, we report asymptotic lower and upper bounds.}
\label{tab:complexity}
\end{table}
Finally, in the binary case, choosing the lower endpoint summary $\underline{\on{AU}}_{\cF}(\cQ)$ recovers the decomposition of \citet{hullermeier2022quantification}.
More precisely, using $\Phi_{\min}(A)=\inf A$ yields the following identities.
\begin{proposition}\label{prop:binary-tv}
Let $\cY=\{0,1\}$ and let $\cQ \subseteq \Delta_1$ be a credal set with
\begin{equation*}
a:=\inf_{p\in\cQ} p(\{0\}),
\quad
b:=\sup_{p\in\cQ} p(\{0\}),
\quad
0\le a\le b\le 1.
\end{equation*}
Then, for $\cF=\cF_{\mathrm{TV}}$,
\begin{equation*}
\begin{aligned}
\on{TU}_{\cF_{\mathrm{TV}}}(\cQ) &= \min\{1-a,\, b\},\\
\underline{\on{AU}}_{\cF_{\mathrm{TV}}}(\cQ) &= \min\{a,\, 1-b\},\\
\on{EU}_{\cF_{\mathrm{TV}}}(\cQ) &= \tfrac12(b-a),
\end{aligned}
\end{equation*}
and in particular
\begin{equation*}
\on{TU}_{\cF_{\mathrm{TV}}}(\cQ)
=
\underline{\on{AU}}_{\cF_{\mathrm{TV}}}(\cQ)
+
2\,\on{EU}_{\cF_{\mathrm{TV}}}(\cQ).
\end{equation*}
\end{proposition}
This shows that the proposed framework extends the binary decomposition to general credal sets and multiclass settings, and makes the summarization step explicit.

\section{Experiments}\label{sec:exp}
Empirically assessing uncertainty quantification is difficult, especially because there is no directly observable ground-truth uncertainty. We therefore conduct selective prediction experiments, where the model may abstain on inputs it deems uncertain. Accuracy–rejection curves (ARCs) \citep{huhn2008fr3} visualize this: sort test points by decreasing uncertainty score, reject the top $r\%$, and compute accuracy on the remaining. For an uncertainty-aware learner that can distinguish certain from less certain cases, accuracy should generally improve as the rejection rate $r$ increases.

\begin{figure*}[p]
    \centering
    \caption{ Accuracy–rejection curves on \dataset{CIFAR-10} \emph{(left)} and \dataset{CIFAR-100} \emph{(right)} for credal set predictors: \textbf{(a)}~Total uncertainty with $\langle \on{TU}_{\cF_{\mathrm{TV}}}\rangle$ and $\langle S^*\rangle$; \textbf{(b)}~Aleatoric uncertainty with $\langle\on{AU}_{\cF_{\mathrm{TV}}}^{\Phi_{\mathrm{int}}}\rangle$, $\langle S_*\rangle$, and $\langle S^* - GH\rangle$; \textbf{(c)}~Epistemic uncertainty with $\langle \on{EU}_{\cF_{\mathrm{TV}}}\rangle$, $\langle S^* - S_*\rangle$, and $\langle GH\rangle$. At rejection rate $r$, test points are sorted by uncertainty and the top $r\%$ most-uncertain are discarded; accuracy is computed on the remainder.  Solid lines show the mean over seeds; shaded bands denote $\pm 1$ s.d. The AUC of each curve ($\uparrow$ higher is better) is printed next to the corresponding label. Our framework is consistently competitive across both datasets.} 
    \includegraphics[width= \textwidth]{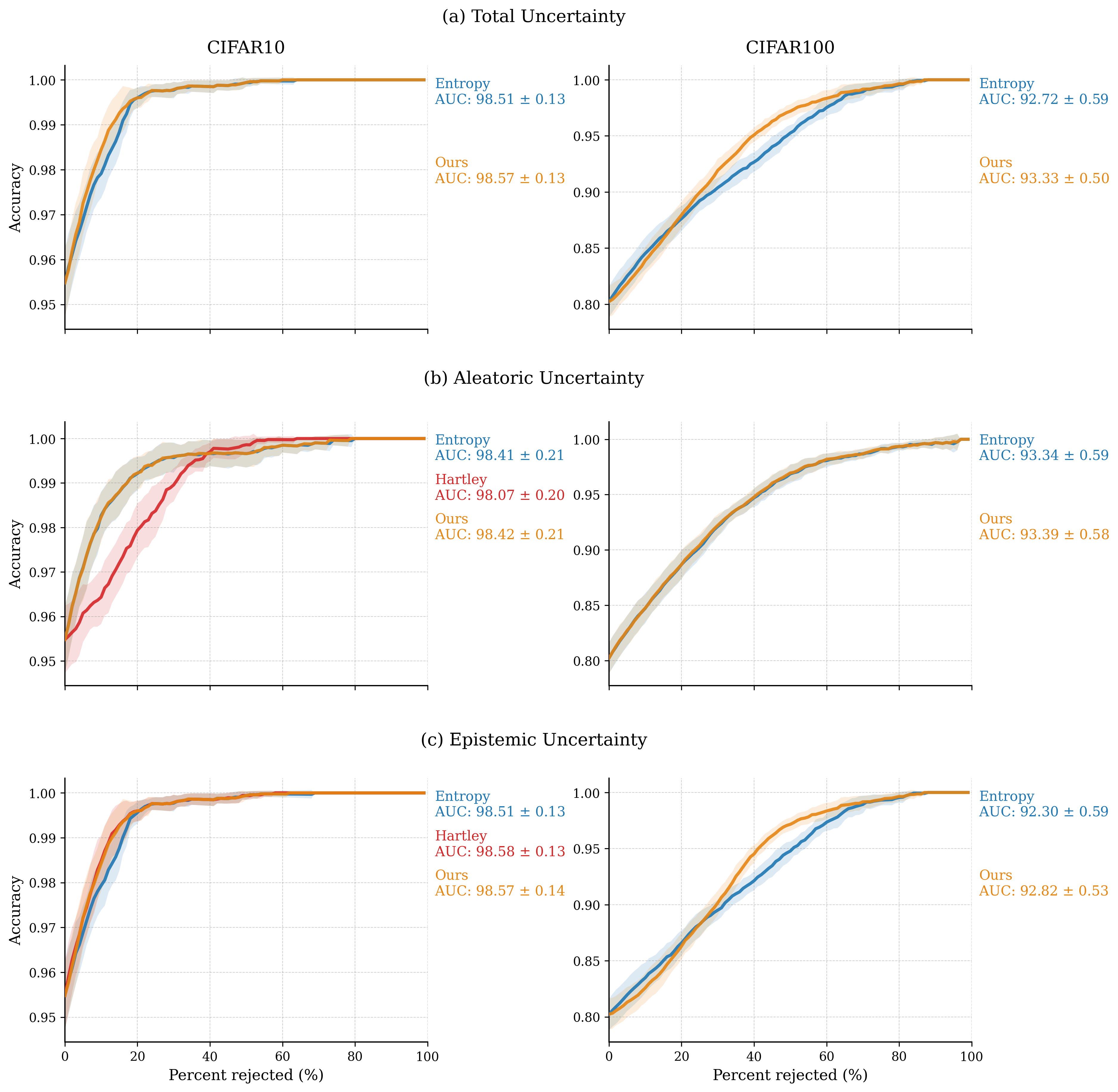}

    \label{fig:ar_cifar}
\end{figure*}

\begin{table*}[t!]
\centering
\caption{ Performance comparison of uncertainty quantification measures across datasets from the KEEL repository. Results are averaged over 64 datasets, 59 with $K \leq 10$ classes and 5 with $K > 10$. The table reports the mean $\pm$ standard deviation of the Area Under the AR Curve (AUC), Monotonicity Ratio (MR) and time (in seconds) for each uncertainty measure. Time refers to the total computation of the uncertainty measure over the entire test set. Bold values show the best mean per component--measure--$K$ group; a method is treated as tied with the best if its mean is no more than 1 SEM away from the best, where  $\mathrm{SEM} = \sigma/\sqrt{N}$ (using the method's own $\sigma$).}  
\label{tab:overall_results}
\begin{tabular}{llcccccc}
\toprule
\multirow{2}{*}{Component} &
\multirow{2}{*}{Measure} &
\multicolumn{3}{c}{Low number of classes ($K \leq 10$, $N = 59$)} &
\multicolumn{3}{c}{High number of classes ($K > 10$, $N = 5$)} \\
\cmidrule(lr){3-5} \cmidrule(lr){6-8}
& & AUC & MR & Time (s) & AUC & MR & Time (s)\\
\midrule
AU & Hartley  & 83.8 $\pm$ 11.5 & 52.4 $\pm$ 18.6 & 18.054 $\pm$ 49.8 & -- & -- & -- \\
   & Entropy  & \textbf{89.2} $\pm$ 9.5  & \textbf{77.6} $\pm$ 15.1 & \textbf{0.001} $\pm$ $\approx$ 0.0 &
          \textbf{92.9} $\pm$ 4.2  & \textbf{89.9} $\pm$ 6.9  & \textbf{0.006} $\pm$ $\approx$ 0.0 \\
   & Ours     & \textbf{89.2} $\pm$ 9.5 & \textbf{77.6} $\pm$ 15.0 & 0.323 $\pm$  1.0 &
          \textbf{92.9} $\pm$ 4.1 & \textbf{90.0} $\pm$ 7.3 & 1.249 $\pm$ 1.8 \\
\midrule
EU & Hartley  & \textbf{87.1} $\pm$ 11.5 & \textbf{65.9} $\pm$ 17.5 & 8.725 $\pm$ 33.0 & -- & -- & -- \\
   & Entropy  & 79.5 $\pm$ 17.0 & 46.7 $\pm$ 23.0 & 9.291 $\pm$ 27.0 &
          84.5 $\pm$ 11.3 & 49.2 $\pm$ 17.1 & 11.669 $\pm$ 14.0 \\
   & Ours & \textbf{86.8} $\pm$ 12.0 & \textbf{64.6} $\pm$ 17.9 & \textbf{0.004} $\pm$ $\approx$ 0.0 &
          \textbf{89.8} $\pm$ 7.1 & \textbf{69.0} $\pm$ 13.8 & \textbf{0.018} $\pm$ $\approx$ 0.0 \\
\midrule
TU & Entropy  & \textbf{88.4} $\pm$ 10.3 & \textbf{73.0} $\pm$ 16.2 & 9.290 $\pm$ 26.567 &
          \textbf{90.6} $\pm$ 5.9 & \textbf{78.1} $\pm$ 10.6 & 11.663 $\pm$ 14.0 \\
   & Ours & \textbf{88.6} $\pm$ 10.2 & \textbf{73.2} $\pm$ 16.2 & \textbf{0.001} $\pm$ 0.003 &
          \textbf{91.0} $\pm$ 5.8 & \textbf{79.8} $\pm$ 9.3 & \textbf{0.002} $\pm$ $\approx$ 0.0 \\
\bottomrule
\end{tabular}
\end{table*} 

We benchmark TV-instantiated measures against entropy and Hartley baselines. 
For aleatoric uncertainty, we use the endpoint summary 
\begin{align*}
\on{AU}^{\Phi_{\mathrm{int}}}_{\cF_{\mathrm{TV}}}(\cQ) = \bigl( \underline{\on{AU}}_{\cF_{\mathrm{TV}}}(\cQ), \overline{\on{AU}}_{\cF_{\mathrm{TV}}}(\cQ) \bigr). 
\end{align*}
Instances are ranked in descending lexicographic order: $\cQ_i$ is considered more aleatorically uncertain than $\cQ_j$ if 
\begin{align*}
\underline{\on{AU}}_{\cF_{\mathrm{TV}}}(\cQ_i) > \underline{\on{AU}}_{\cF_{\mathrm{TV}}}(\cQ_j), 
\end{align*}
or if the lower endpoints are equal and 
\begin{align*}
\overline{\on{AU}}_{\cF_{\mathrm{TV}}}(\cQ_i) > \overline{\on{AU}}_{\cF_{\mathrm{TV}}}(\cQ_j).
\end{align*}
Remaining exact ties are broken deterministically by the original sample order.
This ordering is conservative in the following sense: an instance whose
minimum attainable aleatoric uncertainty is already high is rejected before
one that still admits a low-uncertainty explanation; the upper endpoint is
used only as a tie-breaker.
Implementation details and experimental setup are described in Appendices~\ref{app:impl} and~\ref{app:exp-setup}.
Figure~\ref{fig:ar_cifar} extends \citet{chau2025integral} by reporting ARCs for total, aleatoric, and epistemic uncertainty on \dataset{CIFAR-10} and \dataset{CIFAR-100}, computed from ensemble-induced credal sets. 
Across all three components, our measures are competitive with the baselines on both datasets, often matching or improving their AUC while retaining the computational advantages predicted by Table~1.
Table~\ref{tab:overall_results} presents overall performance on a broader collection of KEEL datasets (detailed in Appendix~\ref{app:exp-setup}). Alongside the AUC, we report the Monotonicity Ratio (MR), defined as the percentage of bins in which the AR curve does not decrease, capturing shape differences AUC may overlook. Competitive performance together with low computation times highlights a key benefit of our approach. Additional experiments are provided in Appendix~\ref{app:add-results}.

\emph{Limitation.} 
We identified a pathological case, though it stems from a representational artifact rather than a deficiency of the measures themselves. If the credal set is an inaccurate representation of the underlying uncertainty, the uncertainty estimates derived from it will inherit this distortion. Although this limitation is shared by any method operating on the same representation, our measures may be more sensitive because they rely on the extreme points of the credal set: a miscalibrated ensemble member can shift the reference point corresponding to zero uncertainty. An illustrative example is given in Appendix~\ref{app:add-results}.

\section{Conclusion}\label{sec:conlc} 
In this paper, we proposed a distance-based framework for credal uncertainty quantification in multiclass classification. It yields principled measures of total, aleatoric, and epistemic uncertainty. We showed that the measures admit clear semantics, satisfy natural desiderata under mild assumptions, and recover established binary measures. For the total variation instantiation, we derived closed-form expressions with favorable complexity. Empirically, our measures are competitive with entropy and Hartley baselines while offering substantial runtime gains. 

A natural direction for future work is to further study alternative IPM  classes, such as Wasserstein or kernel-based variants, as well as extensions beyond classification to structured prediction and decision-making under imprecision. More broadly, while empirical evaluation remains predominant in machine learning, uncertainty measures can also benefit from a systematic  axiomatic analysis. Such an analysis may help clarify which qualitative  properties measures of total, aleatoric, and epistemic uncertainty should  satisfy, and whether these properties can be specified in a task-independent way or must depend on the downstream purpose \citep{hofman2026uncertainty}.


\newpage 
\begin{acknowledgements} 
Xabier Gonzalez-Garcia's and Humberto Bustince's research has been supported by the PID2022-136627NB-I00, funded by MCIN/AEI/10.13039/501100011033/FEDER UE. Yusuf Sale acknowledges support by the DAAD programme Konrad Zuse Schools of Excellence in Artificial Intelligence, sponsored by the Federal Ministry of Research, Technology and Space. Siu Lun Chau was supported by a Start-up Grant from Nanyang Technological University, Singapore. Julian Rodemann acknowledges support by the Bavarian Academy of Sciences and the LMU mentoring program.
\end{acknowledgements}

\bibliography{references}

\newpage
\onecolumn

\title{Supplementary Material}
\maketitle

\appendix
\startcontents[sections]
\printcontents[sections]{l}{1}{\setcounter{tocdepth}{3}}

\clearpage
\section{Related work}\label{app:related}
The idea behind \textit{credal sets} (and the terminology) originates from \cite{levi1974} and was later popularized and given a solid theoretical foundation by \cite{walley1991statistical}. As introduced in Section~\ref{sec:prelim}, credal sets usually---not always---refer to closed and convex sets of probability measures. This is due to the fact that under “coherence” axioms, the representing set of dominating probabilities is automatically closed and convex, see §3.3 in \citet[Chapter 3]{walley1991statistical}.
Notably, credal sets are only one of several imprecise probabilistic models \citep{augustin2014introduction}, alongside capacities \citep{choquet1954theory},
interval probabilities \citep{weichselberger2000theory,poehlmann}, and lower previsions \citep{walley1991statistical,troffaes2014lower}, to name only a few. A credal set makes the distinction between sources of uncertainty explicit: randomness within each probability measure captures aleatoric uncertainty, while the extent of the set itself reflects epistemic uncertainty. Due to their intuitive appeal and their generality, credal sets have received increasing attention in the machine learning and statistics community recently. Methodological applications range from Bayesian neural networks (BNNs) \citep{caprio_IBNN} and optimization \citep{rodemann2024imprecise} to hypothesis testing \citep{chau2025credal} and conformal prediction \citep{caprio2025conformal,caprio2025joys}. 
How to learn credal sets algorithmically has also attracted growing attention in recent years, with several complementary lines of work.
Prominent directions include relative likelihood approaches (e.g., CreRL), which define plausibility via a likelihood-ratio budget and search for a diverse set of high-likelihood hypotheses \citep{lohr2026credal}; wrapper-style methods (Credal Wrapper) that aggregate independently trained models to obtain classwise lower/upper probability bounds and hence box credal sets \citep{wang2025credal}; distance-filtered ensembles (Credal Ensembling) that retain only the closest predictions under a chosen metric \citep{nguyen2025credal}; credal deep ensembles that learn per-class lower/upper bounds with specialized heads and losses \citep{wangCredalIntervalNet2025}; and credal Bayesian deep learning, which aggregates BNNs trained under diverse priors and takes the convex hull of sampled posteriors \citep{caprio2024credalb}.  

A central question for machine learning practitioners is how to extract meaningful summaries of total, aleatoric, and epistemic uncertainty from a credal set that can drive downstream tasks such as selective prediction, out-of-distribution detection, or active learning. Several recent works have addressed this challenge. \citet{hullermeier2022quantification} provided a comparison of existing credal uncertainty measures, identifying issues such as ad hoc additive decompositions and violations of desirable axiomatic properties; notably, the measures they found to be best justified are limited to the binary classification setting, leaving the multiclass case underexplored. \citet{sale2023volume} investigated whether the geometric volume of the credal set is a meaningful measure of epistemic uncertainty, finding it to be of limited effectiveness in multiclass classification. \citet{saleLabelWise2024} proposed a label-wise decomposition of aleatoric and epistemic uncertainty, moving from global scalar summaries toward per-class quantification better suited to the structure of multiclass problems. \citet{hofman2024quantifying} proposed uncertainty measures for credal sets grounded in proper scoring rules. These newer contributions build on a rich foundation in imprecise probability theory and uncertainty quantification \citet{pal1992uncertainty,pal1993uncertainty,moral1992calculating,walley1999upper,abellan2003maximum, abellan2005additivity,abellan2006disaggregated,abellan2006measures,bronevich2008axioms,bronevich2010measures}.  

In contrast to the existing approaches discussed above, our work departs from the current literature in three key respects.
First, we abandon the additive decomposition convention and instead provide independent, semantically grounded definitions for each uncertainty component---total, aleatoric, and epistemic---based on distance-based principles within the IPM framework.
Second, rather than reducing credal aleatoric uncertainty to a single scalar, we characterize it as the \emph{set} of aleatoric uncertainty values attainable across the credal set, thereby preserving the ambiguity encoded by the credal prediction; when a summary is needed for downstream tasks, we provide aggregation strategies that are lossless in a well-defined sense, e.g., the endpoint summary.
Third, we place particular emphasis on computational tractability: the resulting family of measures is designed to scale to modern multiclass classification settings, yielding closed-form expressions with favorable complexity for common IPM choices such as total variation.

\clearpage
\section{Proofs}\label{app:proofs}


\begin{proof}[Proof of Proposition \ref{prop:properties}] 
A1--A2 (Non-negativity, boundedness, and continuity).
Under the assumption  that $\cF$ is a uniform class with respect to weak convergence  (Remark~4.1), $d_{\cF}$ satisfies  
\begin{align*}
0 \le d_{\cF}(p,q)  \le \sup_{f\in\cF}(\sup f-\inf f)<\infty  
\end{align*}
and is jointly continuous in $(p,q)$ \citep[Theorem~4.3]{muller1997integral}. All four  quantities inherit boundedness directly. Continuity is understood with  respect to the Hausdorff topology on $\mathcal V(\Delta_{K-1})$, as  specified in A2. Since $\Delta_{K-1}$ is compact and $d_{\cF}$ is jointly continuous, the relevant objective functions are uniformly continuous. Therefore, taking suprema and infima over Hausdorff-convergent compact sets preserves continuity. Finite minima over $y\in\mathcal Y$ preserve continuity as well, which proves continuity of $\on{TU}_{\cF}$, $\underline{\on{AU}}_{\cF}$,  $\overline{\on{AU}}_{\cF}$, and $\on{EU}_{\cF}$.

A3 (Monotonicity). Let $\cQ \subseteq \cM$. Taking the supremum over a larger set is non-decreasing, so $\on{TU}_\cF(\cQ) \leq \on{TU}_\cF(\cM)$ and $\overline{AU}_{\cF}(\cQ) \leq \overline{AU}_{\cF}(\cM)$. Taking the infimum over a larger set is non-increasing, so $\underline{AU}_{\cF}(\cQ) \geq \underline{AU}_{\cF}(\cM)$. For $\EU$, we use the equivalent form $\EU(\cQ) = \frac{1}{2}\sup_{f \in \cF}\bigl(\sup_{p \in \cQ} \EE_p[f] - \inf_{q \in \cQ} \EE_q[f]\bigr)$: enlarging $\cQ$ to $\cM$ can only increase each inner supremum and decrease each inner infimum, so $\EU(\cQ) \leq \EU(\cM)$.

A4 (Probability consistency). When $\cQ = \{p\}$, $\on{TU}_\cF(\{p\}) = \underline{AU}_{\cF}(\{p\}) = \overline{AU}_{\cF}(\{p\}) = \inf_{y} d_{\cF}(p,\delta_y) = \AU(p)$, and $\EU(\{p\}) = \tfrac{1}{2}\sup_{p,q\in\{p\}} d_{\cF}(p,q) = 0$.

A5 (Extreme point characterization). For fixed $y \in \mathcal{Y}$, define
\begin{align*}
g_y(p) := d_{\cF}(p,\delta_y)
        = \sup_{f\in\cF} \left| \mathbb{E}_p[f] - f(y) \right|.
\end{align*}
Since $p \mapsto \mathbb{E}_p[f]-f(y)$ is affine for every $f\in\cF$,
the map $g_y$ is convex as a pointwise supremum of absolute affine functions.
By the Bauer maximum principle,
\begin{align*}
\sup_{p\in \cQ} g_y(p)
=
\sup_{p\in \operatorname{ext}(\cQ)} g_y(p).
\end{align*}
Taking the finite infimum over $y\in\mathcal{Y}$ gives
\begin{align*}
\on{TU}_{\cF}(\cQ)
=
\inf_{y\in\mathcal{Y}}
\sup_{p\in\operatorname{ext}(\cQ)}
d_{\cF}(p,\delta_y),
\end{align*}
so $\on{TU}_{\cF}$ satisfies A5.

For epistemic uncertainty, define
\begin{align*}
h(p,q) := d_{\cF}(p,q)
        = \sup_{f\in\cF}
          \left| \mathbb{E}_p[f]-\mathbb{E}_q[f] \right|.
\end{align*}
The map $h$ is convex and continuous on the compact convex set $\cQ\times \cQ$.
Moreover, $\operatorname{ext}(\cQ\times \cQ)
= \operatorname{ext}(\cQ)\times\operatorname{ext}(\cQ)$. Applying the Bauer
maximum principle on $\cQ\times \cQ$ yields
\begin{align*}
\sup_{p,q\in \cQ} d_{\cF}(p,q)
=
\sup_{p,q\in\operatorname{ext}(\cQ)}
d_{\cF}(p,q).
\end{align*}
Hence $\on{EU}_{\cF}$ also satisfies A5.
\end{proof}


\begin{proof}[Proof of Corollary \ref{corollary:monotone}]
If $\cQ \subseteq \cM$, then every $p \in \cQ$ also belongs to $\cM$. Hence
\begin{align*}
\{\on{AU}_\cF(p):p\in \cQ\}
\subseteq
\{\on{AU}_\cF(p):p\in \cM\}.
\end{align*}
The endpoint inequalities follow immediately from taking infima and suprema
over nested sets.
\end{proof}


%
\begin{proof}[Proof of Proposition  \ref{prop:tudominance}]
\textit{Epistemic Uncertainty.} For any $p,q \in \cQ$ and any $y \in \cY$, the triangle inequality gives $\dF{p,q} \leq \dF{p,\dy} + \dF{\dy,q}.$ Taking the supremum over $p,q \in \cQ$ and using that the two terms depend on different variables,
$$
\sup_{p,q \in \cQ} \dF{p,q}
\leq \sup_{p \in \cQ} \dF{p,\dy} + \sup_{q \in \cQ} \dF{\dy,q}
= 2\sup_{p \in \cQ} \dF{p,\dy}.
$$
Then, taking $\inf_{y \in \cY}$ on the right-hand side yields
$$
\frac{1}{2}\sup_{p,q \in \cQ} \dF{p,q}
\leq \inf_{y \in \cY}\sup_{p \in \cQ}\dF{p,\dy},
$$
which is precisely $\EU(\cQ)\leq \on{TU}_\cF(\cQ)$.

\textit{Aleatoric Uncertainty.} Since $p \in \cQ$, we have $\dF{p,\dy} \leq \sup_{q \in \cQ} \dF{q,\dy}$ for every $y \in \cY$. Taking $\inf_{y \in \cY}$ on both sides preserves the inequality, yielding $\AU(p) \leq \on{TU}_\cF(\cQ)$.
\end{proof}

 

\begin{proof}[Proof of Proposition \ref{prop:eushift}]
For a fixed $t \in [-1,1]^K$ with $\sum_{y\in\cY} t_y = 0$, define $\cQ' = \{p' \in \ksimplex : p'(\{y\}) = p(\{y\}) + t_y,\; \forall y \in \cY,\; \text{for some } p \in \cQ\}$. Expanding the expectation of $f \in \cF$ under $p'$ gives $\EE_{p'}[f] = \EE_p[f] + \sum_{y} f(y)\, t_y$. Hence, 
$$
\sup_{p'\in\cQ'}\EE_{p'}[f] - \inf_{q'\in\cQ'}\EE_{q'}[f] = \sup_{p\in\cQ}\EE_p[f] - \inf_{q\in\cQ}\EE_q[f].
$$
Halving both sides and taking $\sup_{f\in\cF}$ yields $\EU(\cQ') = \EU(\cQ)$.
\end{proof}



\begin{proof}[Proof of Proposition  \ref{prop:diammmi}]
By Lemma~5 of \citet{chau2025integral}, $\Cint f\,\dd \underline{p}_{\cQ} \leq \inf_{p\in\cQ}\EE_p[f]$ for every $f \in \cF$, with equality when $\underline{p}_{\cQ}$ is 2-monotone. Similarly, by using the asymmetry property \citep{denneberg1994non} $\Cint f\,\dd \overline{p}_{\cQ} = -\Cint(-f)\,\dd \underline{p}_{\cQ}$ yields $\Cint f\,\dd \overline{p}_{\cQ} \geq \sup_{p\in\cQ}\EE_p[f]$. 
Subtracting and taking the supremum over $f\in\cF$ gives 
\begin{align*}
d_{\cF}\bigl(\underline p_Q,\overline p_Q\bigr) = \sup_{f\in\cF} \left( \int_c f\,d\overline p_Q - \int_c f\,d\underline p_Q \right) \ge \sup_{f\in\cF} \left( \sup_{p\in \cQ}\mathbb E_p[f] - \inf_{q\in \cQ}\mathbb E_q[f] \right) = \sup_{p,q\in \cQ} d_{\cF}(p,q). 
\end{align*}
This proves the diameter bound. Dividing both sides by $2$ and using the definition of $\on{EU}_{\cF}$ yields the stated bound on epistemic uncertainty.
\end{proof}


\begin{proof}[Proof of Proposition \ref{prop:closedforms}]
We derive each expression separately for the test class $\cF_{\mathrm{TV}} = \{\mathbbm{1}_A : A \subseteq \cY\}$.

\textit{Total Uncertainty.} The total variation distance between a probability measure $p$ and a Dirac measure $\delta_y$ reduces to
$$
d_{\cF_{\mathrm{TV}}}(p, \delta_y) = \sup_{A \subseteq \cY} |p(A) - \delta_y(A)| = 1 - p(\{y\}).
$$
Indeed, since $\delta_y(A) \in \{0,1\}$ for every $A \subseteq \cY$, the supremum is attained at both $A = \{y\}$ and $A = \{y\}^c$, giving $1 - p(\{y\})$. Substituting into the definition of total uncertainty,
$$
TU_{\cF_{\mathrm{TV}}}(\cQ) 
= \inf_{y \in \cY} \sup_{p \in \cQ} \bigl(1 - p(\{y\})\bigr)
= \inf_{y \in \cY} \bigl(1 - \inf_{p \in \cQ} p(\{y\})\bigr)
= 1 - \sup_{y \in \cY} \underline{p}_{\cQ}(\{y\}).
$$

\textit{Aleatoric Uncertainty.} The same identity $d_{\cF_{\mathrm{TV}}}(p, \delta_y) = 1 - p(\{y\})$ applied to a single probability measure gives
$$
AU_{\cF_{\mathrm{TV}}}(p) = \inf_{y \in \cY} \bigl(1 - p(\{y\})\bigr) = 1 - \sup_{y \in \cY} p(\{y\}).
$$

\textit{Epistemic Uncertainty.} The total variation distance between two probability measures on $\cY$ satisfies
$$
d_{\cF_{\mathrm{TV}}}(p,q) = \sup_{A \subseteq \cY} |p(A) - q(A)| = \frac{1}{2}\sum_{y \in \cY} |p(\{y\}) - q(\{y\})|,
$$
Substituting into~\eqref{eq:eu},
$$
EU_{\cF_{\mathrm{TV}}}(\cQ) 
= \frac{1}{2} \sup_{p,q \in \cQ} d_{\cF_{\mathrm{TV}}}(p,q)
= \frac{1}{2} \sup_{p,q \in \cQ} \frac{1}{2}\sum_{y \in \cY} |p(\{y\}) - q(\{y\})|
= \frac{1}{4} \sup_{p,q \in \cQ} \sum_{y \in \cY} |p(\{y\}) - q(\{y\})|.  
$$ 
\end{proof}
 

\begin{proof}[Proof of Corollary \ref{cor:mmiequiv}]
For any $p, q \in \cQ$ and any $A \subseteq \cY$, 
$\underline{p}_{\cQ}(A) \leq p(A), q(A) \leq \overline{p}_{\cQ}(A)$, 
so $|p(A) - q(A)| \leq \overline{p}_{\cQ}(A) - \underline{p}_{\cQ}(A)$. 
Since $\cY$ is finite, the infimum and supremum defining 
$\underline{p}_{\cQ}(A)$ and $\overline{p}_{\cQ}(A)$ are attained 
for each $A$, so taking the supremum over $p, q \in \cQ$ and 
$A \subseteq \cY$ yields
$$
\sup_{p,q \in \cQ} d_{\cF_{\mathrm{TV}}}(p,q) 
= \sup_{A \subseteq \cY} 
\bigl(\overline{p}_{\cQ}(A) - \underline{p}_{\cQ}(A)\bigr) 
= d_{\cF_{\mathrm{TV}}}(\underline{p}_{\cQ}, \overline{p}_{\cQ}).
$$
Dividing both sides by $2$ gives the result.
\end{proof}


\begin{proof}[Proof of Proposition \ref{prop:AUTVextreme}]
\textit{Lower endpoint.} Similarly to the proof of Proposition~\ref{prop:properties}, the map $p \mapsto 1 - \max_{y} p(\{y\})$ is concave, as the pointwise maximum of affine functions is convex~\citep{boyd2004convex}. By the Bauer minimum principle~\citep{bauer1958minimalstellen}, its infimum over the convex compact set $\cQ$ is attained at an extreme point.

\textit{Upper endpoint.} If $\delta_{y^*}$ is the closest Dirac to every extreme point in $\mathrm{ext}(\cQ)$, then it remains so for every $p \in \cQ$, and thus $\overline{AU}_{\cF_{\mathrm{TV}}}(\cQ) = \sup_{p \in \cQ} (1 - p(\{y^*\})) = 1 - \underline{p}_{\cQ}(\{y^*\})$, which depends only on $\mathrm{ext}(\cQ)$.
\end{proof}


\begin{proof}[Proof of Proposition \ref{prop:binary-tv}]
The result follows by restricting Proposition~\ref{prop:closedforms} to $\cY = \{0,1\}$.

\textit{Total Uncertainty.} From Proposition~\ref{prop:closedforms}, $TU_{\cF_{\mathrm{TV}}}(\cQ) = 1 - \sup_{y \in \cY} \inf_{p \in \cQ} p(\{y\})$. Evaluating for each class:
$$
TU_{\cF_{\mathrm{TV}}}(\cQ) = 1 - \max\{\inf_{p \in \cQ} p(\{0\}), 1 - \sup_{p \in \cQ} p(\{0\})\} = \min\{1-a, b\}.
$$

\textit{Aleatoric Uncertainty.} In the binary case,
$\max\{p(\{0\}), p(\{1\})\} = \max\{p(\{0\}), 1 - p(\{0\})\}$,
so from Proposition~\ref{prop:closedforms},
$$
\underline{AU}_{\cF_{\mathrm{TV}}}(\cQ)
= 1 - \sup_{p \in \cQ} \max\{p(\{0\}), 1 - p(\{0\})\}
= 1 - \max\{b, 1-a\}
= \min\{a, 1-b\}.
$$

\textit{Epistemic Uncertainty.} In the binary case, $|p(\{0\}) - q(\{0\})| = |p(\{1\}) - q(\{1\})|$, so the $\ell_1$ sum in TV reduces to $2|p(\{0\}) - q(\{0\})|$. By Proposition~\ref{prop:closedforms},
$$
EU_{\cF_{\mathrm{TV}}}(\cQ) = \tfrac{1}{4} \sup_{p,q \in \cQ} 2|p(\{0\}) - q(\{0\})| = \tfrac{1}{2}\left(\sup_{p \in \cQ} p(\{0\}) - \inf_{q \in \cQ} q(\{0\})\right) = \tfrac{1}{2}(b - a).
$$

\textit{Additive decomposition.} Direct verification: $\min\{a, 1-b\} + 2 \cdot \tfrac{1}{2}(b-a) = \min\{a, 1-b\} + b - a = \min\{1-a, b\}$.
\end{proof}

\clearpage
\section{Implementation Details}\label{app:impl}
We provide pseudocode for computing both the proposed measures and the baseline methods. Let $\cQ \subseteq \Delta_{K-1}$ be a finitely generated credal prediction derived from an ensemble of $M$ models, i.e., 
\begin{align*}
\cQ=\operatorname{conv}\mathcal G_\cQ, \qquad \mathcal G_\cQ:=\{p^{(1)},\ldots,p^{(M)}\}, 
\end{align*}
where each $p^{(j)}$ is a probability measure over the class labels. The set $\mathcal G_\cQ$ is a generating set of $\cQ$; it need not coincide with $\operatorname{ext}(\cQ)$, since some generators may be redundant. All algorithms below are stated in terms of the supplied generators $\mathcal G_\cQ$. If desired, redundant generators can be removed first, replacing $\mathcal G_\cQ$ by $\operatorname{ext}(\cQ)$, which can only reduce the computational cost.

\subsection*{C.1 \quad Total Uncertainty}

\begin{algorithm}[H]
\caption{Total Uncertainty --- Ours ($\on{TU}_{\cF_{\mathrm{TV}}}$)}
\label{alg:tu_ours}
\begin{algorithmic}[1]
\Require $\mathcal G_Q$
\For{$y = 1, \dots, K$}
    \State $\underline{p}(\{y\}) \gets \min_{j=1,\dots,M} \; p^{(j)}(\{y\})$ \Comment{Lower probability}
\EndFor
\State \Return $1 - \max_{y=1,\dots,K} \; \underline{p}(\{y\})$
\end{algorithmic}
\end{algorithm}

\begin{algorithm}[H]
\caption{Total Uncertainty --- Entropy ($S^*$)}
\label{alg:tu_shannon}
\begin{algorithmic}[1]
\Require $\mathcal G_Q$
\State Solve $\displaystyle w^{*} \gets \arg\max_{w \in \Delta_{M-1}} \; S\!\Bigl(\textstyle\sum_{j=1}^{M} w_j \, p^{(j)}\Bigr)$ \Comment{Convex optimization over $\Delta_{M-1}$}
\State \Return $S\!\bigl(\textstyle\sum_{j=1}^{M} w_j^{*} \, p^{(j)}\bigr)$
\end{algorithmic}
\end{algorithm}

\subsection*{C.2 \quad Aleatoric Uncertainty}

\begin{algorithm}[H]
\caption{Aleatoric Uncertainty --- Ours, lower endpoint ($\underline{\on{AU}}_{\cF_{\mathrm{TV}}}$)}
\label{alg:au_ours}
\begin{algorithmic}[1]
\Require $\mathcal G_Q$
\For{$j = 1, \dots, M$}
    \State $AU_j \gets 1 - \max_{y=1,\dots,K} \; p^{(j)}(\{y\})$
\EndFor
\State \Return $\min_{j=1,\dots,M} \; AU_j$ \Comment{Lower endpoint of $AU_{\cF_{\mathrm{TV}}}(\cQ)$}
\end{algorithmic}
\end{algorithm}

\begin{algorithm}[H]  
\caption{Aleatoric Uncertainty --- Ours, upper endpoint ($\overline{\on{AU}}_{\cF_{\mathrm{TV}}}$)}
\label{alg:au_upper_ours}
\begin{algorithmic}[1]
\Require $\mathcal G_Q$
\State $y^*_j \gets \arg\max_{y=1,\dots,K} p^{(j)}(\{y\})$ for each $j = 1,\dots,M$
\If{$y^*_1 = y^*_2 = \cdots = y^*_M$} \Comment{All models agree on argmax}
    \State \Return $\max_{j=1,\dots,M}\; \bigl(1 - \max_{y=1,\dots,K} p^{(j)}(\{y\})\bigr)$
\Else \Comment{Models disagree $\rightarrow$ solve Linear programming}
    \State Solve $t^{*} \gets \min_{\substack{w \in \Delta_{M-1} \\ t \geq 0}} \; t
        \quad \text{s.t.} \quad
        \textstyle\sum_{j=1}^{M} w_j\,p^{(j)}(\{y\}) \leq t \;\; \forall\, y = 1,\dots,K$
    
    \State \Return $1 - t^{*}$
\EndIf
\end{algorithmic}
\end{algorithm}

\begin{algorithm}[H]
\caption{Aleatoric Uncertainty --- Entropy ($S_*$)}
\label{alg:au_shannon}
\begin{algorithmic}[1]
\Require $\mathcal G_Q$
\For{$j = 1, \dots, M$}
    \State $S_j \gets -\sum_{y=1}^{K} p^{(j)}(\{y\}) \log_2 p^{(j)}(\{y\})$ \Comment{Shannon entropy with convention $0 \log 0 := 0$}
\EndFor
\State \Return $\min_{j=1,\dots,M} \; S_j$ \Comment{Lower entropy}
\end{algorithmic}
\end{algorithm}

\begin{algorithm}[H]
\caption{Aleatoric Uncertainty --- Hartley residual ($S^\ast - GH$)}
\label{alg:au_hartley}
\begin{algorithmic}[1]
\Require $\mathcal G_Q$
\State Compute $S^\ast$ via Algorithm~\ref{alg:tu_shannon}
\Comment{Upper entropy}
\State Compute $GH$ via Algorithm~\ref{alg:eu_hartley} \Comment{Generalized Hartley}
\State \Return $S^\ast - GH$
\end{algorithmic}
\end{algorithm}

\subsection*{C.3 \quad Epistemic Uncertainty}

\begin{algorithm}[H]
\caption{Epistemic Uncertainty --- Ours ($\on{EU}_{\cF_{\mathrm{TV}}}$)}
\label{alg:eu_ours}
\begin{algorithmic}[1]
\Require $\mathcal G_Q$
\State $d_{\max} \gets 0$
\For{$j = 1, \dots, M-1$} \Comment{Upper-triangular pairs}
    \For{$i = j+1, \dots, M$}
        \State $d \gets \sum_{y=1}^{K} |p^{(j)}(\{y\}) - p^{(i)}(\{y\})|$ \Comment{$l_1$ distance}
        \State $d_{\max} \gets \max(d_{\max},\; d)$
    \EndFor
\EndFor
\State \Return $\tfrac{1}{4}\, d_{\max}$
\end{algorithmic}
\end{algorithm}

\begin{algorithm}[H]
\caption{Epistemic Uncertainty --- Entropy ($S^* - S_*$)}
\label{alg:eu_shannon}
\begin{algorithmic}[1]
\Require $\mathcal G_Q$
\State Compute $S^\ast$ via Algorithm~\ref{alg:tu_shannon}
\Comment{Upper entropy}
\State Compute $S_\ast$ via Algorithm~\ref{alg:au_shannon}
\Comment{Lower entropy}
\State \Return $S^\ast - S_\ast$ 
\end{algorithmic}
\end{algorithm}

\begin{algorithm}[H]
\caption{Epistemic Uncertainty --- Hartley ($GH$)}
\label{alg:eu_hartley}
\begin{algorithmic}[1]
\Require Lower probability of $\cQ$: $\underline{p}_{\cQ}(A)$ for all $A \subseteq \cY$
\For{each $\emptyset\neq A\subseteq\mathcal Y$} \Comment{M\"{o}bius transform of $\underline{p}_{\cQ}$}
\State $m(A) \gets \sum_{B \subseteq A} (-1)^{|A \setminus B|} \; \underline{p}_{\cQ}(B)$
\EndFor
\State \Return $\sum_{\emptyset\neq A\subseteq\mathcal Y} m(A)\log_2 |A|$
\end{algorithmic}
\end{algorithm}

\clearpage
  
\section{Experimental Setup}\label{app:exp-setup}

The experimental setup follows that of  \cite{chau2025integral}. We construct the credal set from an ensemble of 10 classifiers, and predictions are obtained by voting among the models and breaking ties with the average probabilities. In Figure \ref{fig:ar_cifar}, each classifier generates predictions over 10 test data batches, which are assigned uncertainty values by each measure, and Accuracy Rejection (AR) curves illustrate the measures’ performance. In AR curves, predictions with the highest uncertainty are progressively rejected, and the accuracy on the remaining samples is recorded. By plotting the accuracy against the fraction of rejected examples, we visualize how well the uncertainty measure separates easy (high-confidence) from hard (low-confidence) predictions. A steeper curve indicates a more informative uncertainty measure, as rejecting uncertain examples leads to a faster increase in accuracy. Results are averaged over batches to provide a robust evaluation. The experimental setup runs with AR curves discretized into 30 equal-width bins.

Beyond visual comparison, AR performance is commonly summarized through the Area Under the Curve (AUC), which provides a compact measure of overall ranking quality. However, its expressiveness is inherently limited, as it fails to account for the shape of the curve: two AR curves may yield identical areas while exhibiting different behaviors. Since, by construction, accuracy should not decrease as increasingly certain samples are retained, we complement the AUC with a shape-sensitive indicator of monotonicity. Specifically, we define the Monotonicity Ratio (MR) as the proportion of bins in which the AR curve does not decrease. Intuitively, high MR indicates that rejecting uncertain samples never harms accuracy.

\paragraph{Monotonicity Ratio (MR).}
Let $\mathbf{s} = (s_1, \dots, s_B)$ denote an AR curve evaluated over $B$ bins. 
We define the successive differences as
\[
\Delta_i = s_{i+1} - s_i, \quad i = 1, \dots, B-1.
\]

The \textit{Monotonicity Ratio} (MR) is defined as
\begin{equation}
\text{MR} = \frac{1}{B-1} \sum_{i=1}^{B-1} \mathbbm{1}_{\{\Delta_i \ge 0\}},
\end{equation}

where $\mathbbm{1}_{\{\cdot\}}$ denotes the indicator function. 

Table~\ref{tab:overall_results} complements the visual analysis by reporting average metrics across all considered datasets. We next describe the datasets, models, and computational resources used in the experiments.

\subsection{Datasets}

The experiments in this work are conducted on the \dataset{CIFAR-10} and \dataset{CIFAR-100} image datasets, as well as on tabular classification datasets from the \dataset{KEEL} repository. In Appendix \ref{app:add-results}, we use the Fashion-MNIST and SVHN datasets from PyTorch.

\datasetlabel{CIFAR datasets}
The \dataset{CIFAR} datasets, comprising \dataset{CIFAR-10} and \dataset{CIFAR-100} \cite{krizhevsky2009learning}, are widely used benchmarks in computer vision for evaluating image classification algorithms. Both datasets contain color images of size 32×32 pixels, split into 50,000 training images and 10,000 test images. The small image size and relatively limited number of samples per class make both datasets particularly challenging. \dataset{CIFAR-10} comprises 10 broad object classes, whereas \dataset{CIFAR-100} includes 100 more fine-grained classes organized into 20 superclasses.

\datasetlabel{KEEL repository}
Knowledge Extraction based on Evolutionary Learning (KEEL) \cite{AlcalFdez2011KEELDS} is an open-source (GPLv3) Java software tool maintained by the SCI2S research group at the University of Granada \footnote{\url{https://sci2s.ugr.es/keel/datasets.php}}. It provides a comprehensive platform for a wide range of knowledge discovery and data mining tasks. Additionally, it includes a dataset repository covering diverse learning paradigms such as supervised classification, regression, time series analysis, and unsupervised learning, including specialized types such as imbalanced datasets, multi-instance, multi-label, semi-supervised, and low-quality data. Table \ref{tab:KEEL_datasets} reports the 64 datasets that we considered for the work. A problem with more than ten classes is considered a high-class setting. This results in a total of 59 datasets with a low number of classes ($K \leq 10$) and 5 with a high number of classes ($K > 10$). For each dataset, the experiment is repeated over 10 random stratified splits of the data (batches), with models trained and validated on 70\% of the samples and the remaining 30\% used to evaluate the measures.
 
\begin{table}[hbt]
\centering

\caption{
Summary of the 64 datasets from the
KEEL repository, grouped by number of classes:
59 with $K \leq 10$ and 5 with $K > 10$. Preprocessing removed instances with missing values and numerically encoded categorical attributes.
}
\setlength{\tabcolsep}{4pt}
\renewcommand{\arraystretch}{1.1}
\begin{tabularx}{\linewidth}{l Y c c}
\toprule
\textbf{Category} & \textbf{Datasets} & \textbf{Samples (range)} & \textbf{Features (range)} \\
\midrule
\textbf{$K \leq 10$ (59 datasets)} &
adult; appendicitis; australian; automobile; balance; banana; bands; breast; bupa; car; chess; cleveland; coil2000; connect-4; contraceptive; crx; dermatology; fars; flare; german; glass; haberman; hayes-roth; heart; hepatitis; housevotes; ionosphere; iris; led7digit; magic; mammographic; marketing; monk-2; newthyroid; optdigits; page-blocks; penbased; phoneme; pima; ring; saheart; satimage; segment; shuttle; sonar; spambase; spectfheart; splice; tae; thyroid; tic-tac-toe; titanic; twonorm; vehicle; wdbc; wine; wisconsin; yeast; zoo&
80--100{,}968 & 2--85 \\
\textbf{$K > 10$ (5 datasets)} &
kr-vs-k; letter; movement-libras; texture; vowel &
360--28,056  & 6--90 \\
\bottomrule
\end{tabularx}

\label{tab:KEEL_datasets}
\end{table}

\datasetlabel{Other Vision datasets}
PyTorch provides an open-access repository of computer vision datasets with predefined training and test splits.\footnote{\url{https://docs.pytorch.org/vision/0.12/datasets.html}} In the experiments reported in Appendix~\ref{app:add-results}, we consider the Fashion-MNIST and SVHN datasets.
Fashion-MNIST is a more challenging alternative to the original MNIST dataset, consisting of 60,000 training samples and 10,000 test samples. Images are grayscale with a resolution of $28 \times 28$ pixels and are categorized into 10 classes corresponding to different types of clothing.
SVHN (Street View House Numbers) contains color images of house numbers extracted from Google Street View. Images have a resolution of $32 \times 32$ pixels, and the dataset includes 26,032 test samples spanning 10 digit classes.

\subsection{Models}

We employ pretrained neural network models for the \dataset{CIFAR-10} and \dataset{CIFAR-100} datasets, available at https://github.com/chenyaofo/pytorch-cifar-models. These models differ in size and architecture, introducing diversity into the ensemble.

For the tabular KEEL datasets (Table~\ref{tab:overall_results}), 10 random forest models are trained per dataset with randomly selected hyperparameters, following the approach in \cite{chau2025integral}. 

In Appendix~\ref{app:add-results}, we adopt a Vision Transformer (ViT-Base/16) model with 86 million parameters, pretrained on ImageNet-21K.\footnote{\url{https://huggingface.co/google/vit-base-patch16-224}}
All input images are resized to $224 \times 224 $ pixels and normalized using the standard preprocessing pipeline.

\subsection{Compute Resources}
The runtimes reported in Table~\ref{tab:overall_results} were obtained using an AMD Ryzen 7 5800X CPU (8 cores) running at 4.20 GHz.
The additional experiment presented in Appendix~\ref{app:add-results} was conducted on a proprietary computing cluster equipped with four NVIDIA A100 GPUs. Inference for the CIFAR models can be performed on a standard desktop GPU.

\clearpage
\section{Additional Experimental Results}\label{app:add-results}

This appendix extends the empirical validation of Section~\ref{sec:exp}. We evaluate the proposed measures in a selective prediction setting using Vision Transformer models, a state-of-the-art architecture for image classification, together with a credal predictor based on the relative-likelihood method of \citet{lohr2026credal}. Beyond performance validation, we reproduce and analyze in detail the pathological case identified in Section~\ref{sec:exp}.

\subsection{Credal Set Construction via Relative Likelihood}
Let $\{(x_i,y_i)\}_{i=1}^N$ be a dataset formed by i.i.d. samples. The likelihood of a predictor $p$ is defined as
\begin{equation}
    L(p)=\prod_{i=1}^N p(y_i\mid x_i).
\end{equation}
Let $p^\star$ denote the empirical maximum likelihood estimator (MLE) over the considered model class. The relative likelihood of a predictor $p$ is then given by
\begin{equation}
    R(p)=\frac{L(p)}{L(p^\star)}.
\end{equation}
The relative-likelihood approach of \citet{lohr2026credal} constructs a credal set by selecting all predictors whose relative likelihood exceeds a threshold $\alpha$. For experimental purposes (see Section~\ref{sec:path_case}), we discretise the relative-likelihood scale using a finite set
\[
\mathcal{A}=\{0.1,0.2,\ldots,1.0\}.
\]
For each $\alpha \in \mathcal{A}$, we select a predictor $p^\alpha$ such that $R(p^\alpha)\geq \alpha$. The resulting credal predictor at an input $x$ is defined as the convex hull
\begin{equation}
    \cQ_x = \mathrm{conv}
    \left\{p^\alpha(\cdot \mid x) : \alpha \in \mathcal{A} \right\}.
\end{equation}
In this way, even when using the same model architecture, the credal predictor reflects imprecision arising from a set of predictors corresponding to different states of knowledge. Each predictor is initialized using TOBIAS~\citep{lohr2026credal}.

Figure~\ref{fig:ar_hist_mle} shows that the proposed measures are again competitive with the baselines on both \dataset{Fashion-MNIST} and \dataset{SVHN}. In the entropy and Hartley frameworks, the component derived from the additive decomposition tends to underperform, whereas all three components of our framework yield monotonic AR curves with slightly higher AUC values overall.

\subsection{Pathological Case} \label{sec:path_case}

We now augment the ensemble with an additional member $p^{0}$. Since this threshold imposes no constraint on the likelihood, $p^{0}$ is free to be any model regardless of its fit to the data. We set $p^{0}$ to predict a fixed class $y^*$ for every input $x$, i.e., $p^{0}(\cdot \mid x) = \delta_{y^*}$. This model produces predictions entirely disconnected from the training data, corrupting the credal set and, consequently, any uncertainty quantification derived from it. As observed in Figure~\ref{fig:ar_hist_mle_0}, the proposed measures appear more sensitive to this pathological case than the baselines. We next analyze the reason behind this behavior.

When a Dirac measure $\delta_{y^*}$ belongs to $\cQ$, we have $\underline{p}_{\cQ}(\{y\}) = 0$ for all $y \neq y^*$, since $\delta_{y^*}(\{y\}) = 0$ for $y \neq y^*$. The closed-form expression in Proposition ~\eqref{prop:closedforms} then reduces to $TU_{\cF_{\mathrm{TV}}}(\cQ) = 1 - \underline{p}_{\cQ}(\{y^*\})$.

The total uncertainty measure becomes anchored entirely to the class $y^*$, regardless of whether $y^*$ is the correct label. Even when the majority of plausible models in $\cQ$ agree on a different class with high confidence, the single Dirac measure suffices to redirect the reference vertex from the consensus class to $y^*$, distorting the semantic interpretation of the measure. A similar effect propagates to the epistemic uncertainty. The maximal diameter of $\cQ$ becomes dominated by the pair $(\delta_{y^*}, p)$ for whichever $p \in \cQ$ is furthest from $\delta_{y^*}$. 
This explains the inverted AR curves in Figure~\ref{fig:ar_hist_mle_0}: since the models with higher $\alpha$ assign near zero probability to $y^*$ for most instances, $\underline{p}_{\cQ}(\{y^*\}) \approx 0$ and $TU, 2 \times EU \approx 1$ almost everywhere, as the cumulative distributions confirm. The few instances with slightly lower uncertainty (e.g., $TU,2 \times EU \approx 0.99$) are those where high-likelihood models place some residual mass on $y^*$---paradoxically indicating lower model confidence, yet ranked as less uncertain.

\begin{figure}[htb]
    \centering
    \caption{ Selective prediction for credal sets constructed via relative likelihood, on \dataset{FASHION-MNIST} and \dataset{SVHN}. \textbf{Left panels:} Accuracy--rejection (AR) curves for (a) total uncertainty, (b) aleatoric uncertainty, and (c) epistemic uncertainty. Area Under the Curve (AUC, $\uparrow$ higher is better) is reported in each legend. \textbf{Right panels:} Cumulative distribution of uncertainty scores across all test instances. The vertical dashed line marks the median uncertainty value.} 
    \includegraphics[width=0.48\textwidth]{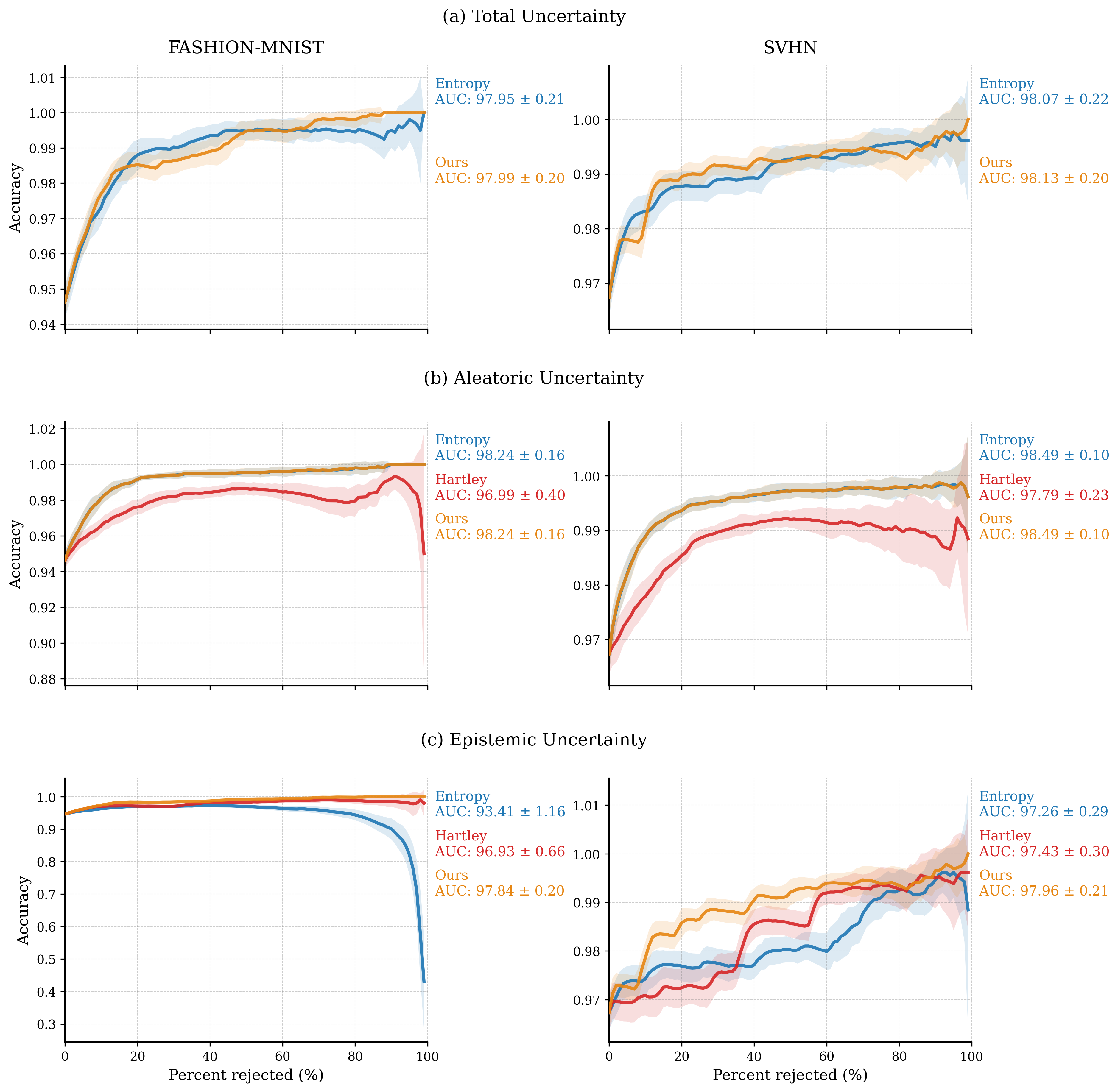}\hfill
    \includegraphics[width=0.48\textwidth]{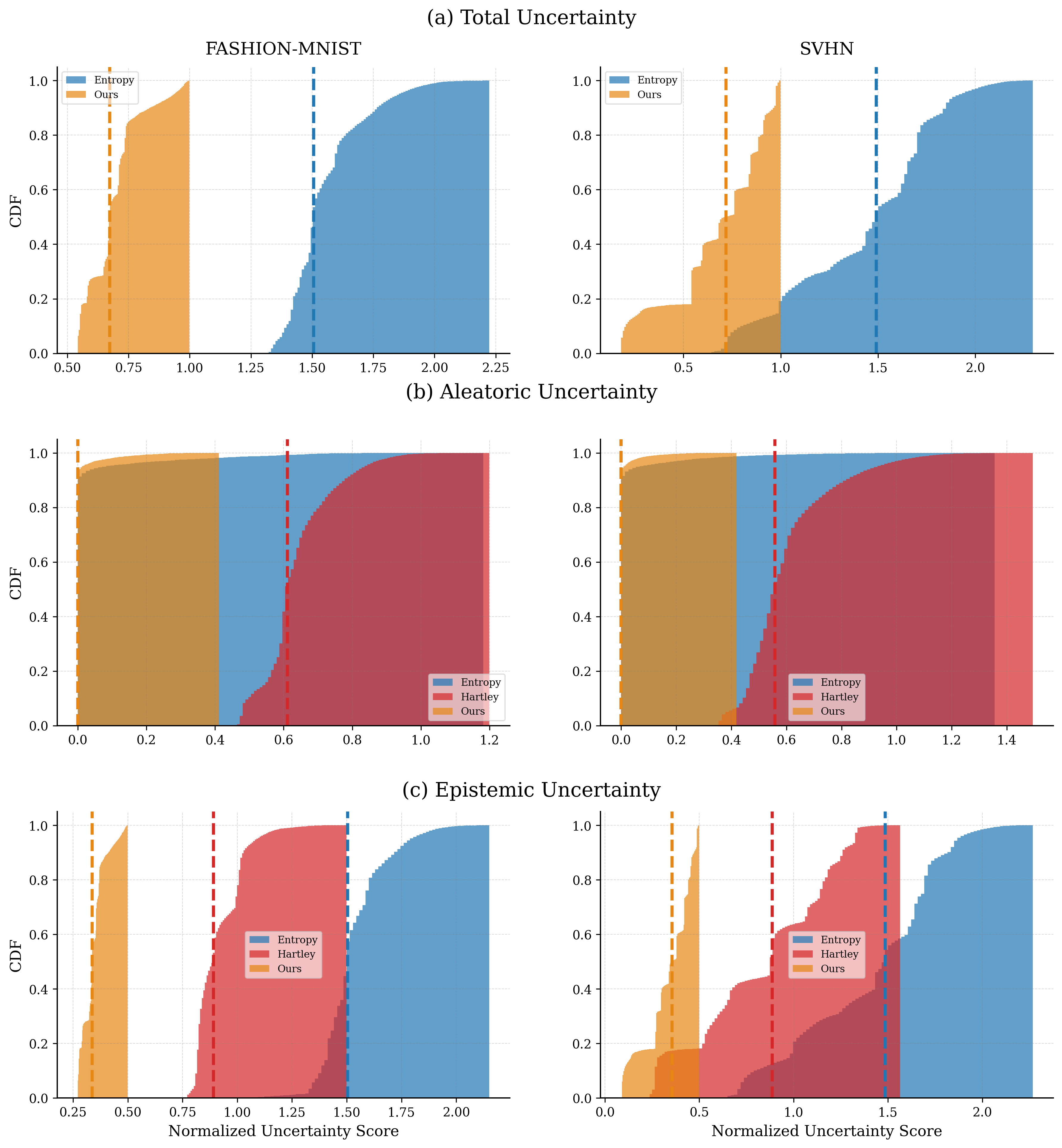}
    
    \label{fig:ar_hist_mle}
\end{figure} 

\begin{figure}[htb]
    \centering
    \caption{Selective prediction for credal sets constructed via relative likelihood, \textbf{including} $p^{0}$, on \dataset{FASHION-MNIST} and \dataset{SVHN}. \textbf{Left panels:} Accuracy--rejection (AR) curves for (a) total uncertainty, (b) aleatoric uncertainty, and (c) epistemic uncertainty. Area Under the Curve (AUC, $\uparrow$ higher is better) is reported in each legend. \textbf{Right panels:} Cumulative distribution of uncertainty scores across all test instances. The vertical dashed line marks the median uncertainty value. The inclusion of $p^{0}$---a deterministic model assigning probability 1 to a fixed class---causes pathological behavior.}
    \includegraphics[width=0.48\textwidth]{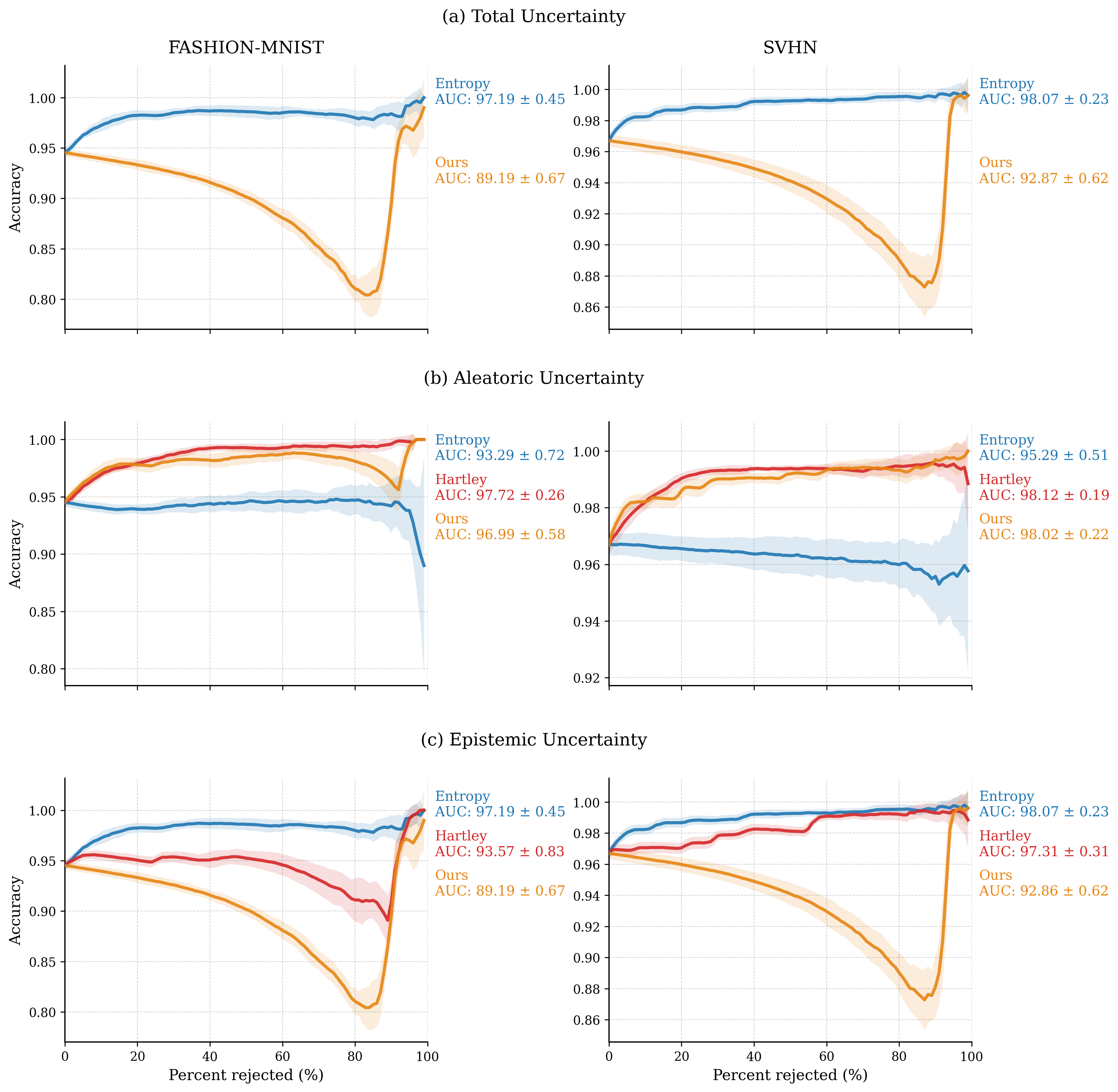}\hfill
    \includegraphics[width=0.48\textwidth]{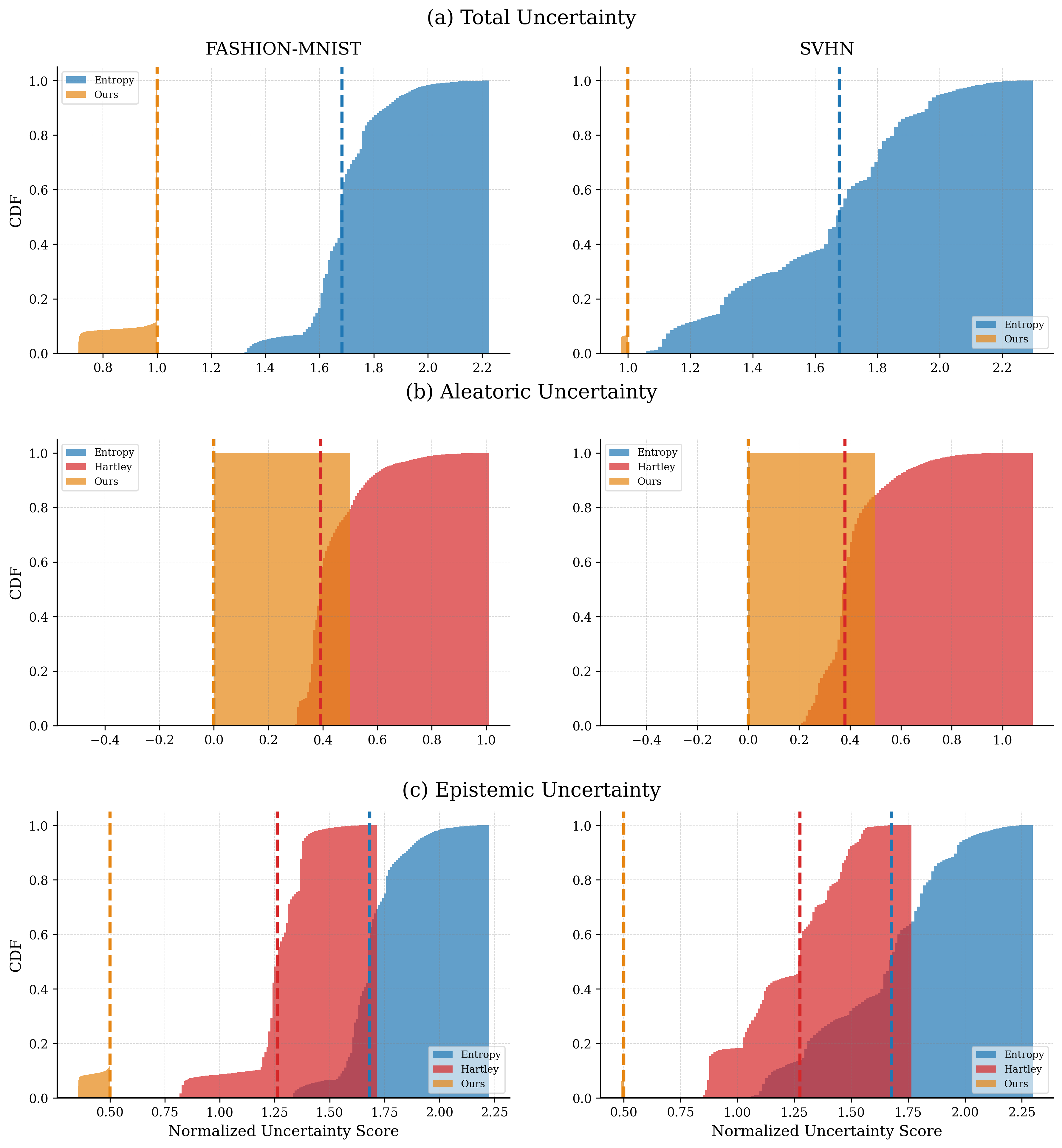}

    \label{fig:ar_hist_mle_0}
\end{figure}

\end{document}